\documentclass{article}

 \PassOptionsToPackage{numbers, compress}{natbib}
 \usepackage[final]{nips_2017}

\usepackage[utf8]{inputenc} 
\usepackage[T1]{fontenc}    
\usepackage{hyperref}       
\usepackage{url}            
\usepackage{booktabs}       
\usepackage{amsfonts}       
\usepackage{nicefrac}       
\usepackage{microtype}      
\usepackage{amssymb}
\usepackage{pifont}
\newcommand{\cmark}{\ding{51}}%
\newcommand{\xmark}{\ding{55}}%

\usepackage{amsmath,amssymb,xspace,amsthm,Tabbing,bm}
\usepackage{color,mathtools}
\usepackage{amsthm}
\usepackage{listings}
\usepackage{multirow}
\usepackage{fullpage}
\usepackage{bbm}
\usepackage[export]{adjustbox}
\usepackage{wrapfig}
\usepackage{tikz}
\usepackage{wasysym}

\usepackage{calrsfs}

\DeclareMathAlphabet{\pazocal}{OMS}{zplm}{m}{n}

\usepackage{times}
\usepackage{graphicx} 

\usepackage{natbib}

\usepackage{graphicx}
\usepackage{caption}
\usepackage{subcaption}

\usepackage{algorithm}
\usepackage{algorithmic}
\newtheorem{theorem}{Theorem}

\newtheorem{lemma}{Lemma}

\newtheorem{remark}{Remark}

\providecommand{\nor}[1]{\left\lVert {#1} \right\rVert}
\providecommand{\abs}[1]{\lvert{#1}\rvert}

\providecommand{\scalT}[2]{\left\langle{#1},{#2}\right\rangle}

\def\bit{\begin{itemize}}
\def\eit{\end{itemize}}
\def\it{\item}

\def\ben{\begin{enumerate}}
\def\een{\end{enumerate}}

\def \P{\mathbb{P}}
\def \Q{\mathbb{Q}}

\definecolor{ao(english)}{rgb}{0.0, 0.5, 0.0}

\usepackage{listings}
\usepackage{color}

\definecolor{dkgreen}{rgb}{0,0.6,0}
\definecolor{gray}{rgb}{0.5,0.5,0.5}
\definecolor{mauve}{rgb}{0.58,0,0.82}

\usepackage{enumitem}

\newif\ifarxiv 
\arxivtrue 

\title{Fisher GAN}

\author{
Youssef Mroueh$^*$, Tom Sercu$^*$ \\
    \texttt{mroueh@us.ibm.com}, 
   \texttt{tom.sercu1@ibm.com} \\
   $*$ Equal Contribution\\
   AI Foundations, IBM Research AI  \\
   IBM T.J Watson Research Center \\
}

\begin{document}

\maketitle

\begin{abstract}
Generative Adversarial Networks (GANs) are powerful models for learning complex distributions.
Stable training of GANs has been addressed in many recent works which explore different metrics between distributions. In this paper we introduce \emph{Fisher GAN} which fits within the Integral Probability Metrics (IPM) framework for  training GANs.
Fisher GAN defines a critic with a data dependent constraint on its \emph{second order moments}.
We show in this paper that Fisher GAN allows for stable and time efficient  training that does not compromise the capacity of the critic, and does not need data independent constraints such as weight clipping.
We analyze our Fisher IPM theoretically and provide an algorithm based on Augmented Lagrangian for Fisher GAN.
We validate our claims on both image sample generation and semi-supervised classification using Fisher GAN.

\end{abstract}

\section{Introduction}

Generative Adversarial Networks (GANs) \cite{goodfellow2014generative} have recently become a prominent method to learn high-dimensional probability distributions.
The basic framework consists of a generator neural network which learns to generate samples which approximate the distribution, while the discriminator measures the distance between  the real data distribution, and this learned distribution that is referred to as \emph{fake} distribution.
The generator uses the gradients from the discriminator to minimize the distance with the real data distribution.
The distance between these distributions  was the object of study in \cite{arjovsky2017towards}, and highlighted the impact of the distance choice on the stability of the optimization.
The original GAN formulation optimizes the Jensen-Shannon divergence,
while later work generalized this to optimize 
f-divergences \cite{nowozin2016f}, KL \cite{kaae2016amortised}, 
the Least Squares objective \cite{mao2016least}. 
Closely related to our work, Wasserstein GAN (WGAN) \cite{arjovsky2017wasserstein} uses the earth mover distance, for which the discriminator function class needs to be constrained to be Lipschitz. To impose this Lipschitz constraint, WGAN proposes to use \emph{weight clipping}, i.e. a data independent constraint, but this comes at the cost of reducing the capacity of the critic and high sensitivity to the choice of the clipping hyper-parameter.
A recent development Improved Wasserstein GAN (WGAN-GP) \cite{gulrajani2017improved} introduced a data dependent constraint namely a  \emph{gradient penalty} to enforce the Lipschitz constraint  on the critic, which does not compromise the capacity of the critic but comes at a high computational cost.

We build in this work on the Integral probability Metrics (IPM) framework for learning GAN of \cite{mroueh2017mcgan}. Intuitively the IPM defines a \emph{critic} function $f$, that maximally discriminates between the real and fake distributions.
We propose a theoretically sound and time efficient data dependent constraint on the critic of Wasserstein GAN, that allows a stable training of GAN and does not compromise the capacity of the critic. 
Where WGAN-GP uses a penalty on the gradients of the critic, \emph{Fisher GAN} imposes a constraint on the \emph{second order moments} of the critic.
This extension to the IPM framework is inspired by the \emph{Fisher Discriminant Analysis} method.

The main contributions of our paper are:
\begin{enumerate}[wide,labelwidth=!,labelindent=0pt,topsep=0pt]
\item We introduce in Section \ref{sec:fisherIPM} the Fisher IPM, a scaling invariant distance between distributions.
  Fisher IPM introduces a data dependent constraint on the second order moments of the critic that discriminates between the two distributions.
  Such a constraint ensures the boundedness of the metric and the critic.
  We show in Section \ref{sec:whiteFisher} that Fisher IPM when approximated with neural networks, corresponds to a discrepancy between \emph{whitened} mean feature embeddings of the distributions.
  In other words a mean feature discrepancy that is measured with a Mahalanobis distance in the space computed by the neural network.
\item  We show  in Section \ref{sec:theory} that Fisher IPM corresponds to the \emph{Chi-squared} distance ($\chi_2$) when the critic has unlimited capacity (the critic belongs to a universal hypothesis function class).
  Moreover we prove in Theorem 2 that even when the critic is parametrized by a neural network, it approximates the $\chi_2$ distance with a factor which is a inner product between optimal and neural network critic.
  We finally derive generalization bounds of the learned critic from samples from the two distributions, assessing the statistical error and its convergence to the Chi-squared distance from finite sample size.
\item We use Fisher IPM as a GAN objective \footnote{Code is available at \url{https://github.com/tomsercu/FisherGAN}}
  and formulate an algorithm that combines desirable properties (Table \ref{tab:comparison}): 
  a stable and meaningful loss between distributions for GAN as in Wasserstein GAN \cite{arjovsky2017wasserstein}, 
  at a low computational cost similar to simple weight clipping, 
  while not compromising the  capacity  of the critic via a data dependent constraint but at a much lower computational cost than \cite{gulrajani2017improved}. 
  Fisher GAN achieves strong semi-supervised learning results without need of batch normalization in the critic.
\end{enumerate}

\begin{table}[H]
\vskip -0.2 in
\begin{center}
\caption{Comparison between Fisher GAN and recent related approaches.
  \label{tab:comparison} }
\begin{tabular}{| l | l | l | l | l | }
\hline
              & Stability & Unconstrained  &Efficient  & Representation   \\
              &  & capacity  &Computation &power (SSL)  \\
\hline
Standard GAN \cite{goodfellow2014generative,salimans2016improved} & ~~~~\textcolor{red}{\xmark}     &~~~~\textcolor{ao(english)}{\cmark}                     & ~~~~\textcolor{ao(english)}{\cmark}     & ~~~~\textcolor{ao(english)}{\cmark}     \\
WGAN, McGan  \cite{arjovsky2017wasserstein,mroueh2017mcgan}        & ~~~~\textcolor{ao(english)}{\cmark}     &~~~~\textcolor{red}{\xmark}                   &~~~~\textcolor{ao(english)}{\cmark}    &~~~~\textcolor{red}{\xmark}        \\
WGAN-GP \cite{gulrajani2017improved}     &~~~~\textcolor{ao(english)}{\cmark}      & ~~~~\textcolor{ao(english)}{\cmark}                    &~~~~\textcolor{red}{\xmark}    &~~~~?         \\
\hline
Fisher Gan (Ours) & ~~~~\textcolor{ao(english)}{\cmark} &~~~~\textcolor{ao(english)}{\cmark}                    & ~~~~\textcolor{ao(english)}{\cmark}     & ~~~~\textcolor{ao(english)}{\cmark}      \\
\hline
\end{tabular}
\end{center}
\vskip -0.2 in
\end{table}

\section{Learning GANs with Fisher IPM}\label{sec:fisherIPM}
\subsection{Fisher IPM in an arbitrary function space: General framework}
\label{sec:fisheripm}
\textbf{Integral Probability Metric (IPM).} 
Intuitively an IPM defines a critic function $f$ belonging to a function class $\mathcal{F}$, that maximally discriminates between two distributions. The function class $\mathcal{F}$ defines how $f$ is bounded, which is crucial to define the metric.
More formally, consider a compact space $\pazocal{X}$ in $\mathbb{R}^{d}$. Let $\mathcal{F}$ be a set of \emph{measurable, symmetric and bounded} real valued  functions on $\pazocal{X}$. Let $\mathcal{P}(\pazocal{X})$ be the set of measurable probability distributions on $\pazocal{X}$. Given two probability distributions $\mathbb{P},\mathbb{Q} \in \mathcal{P}({\pazocal{X}})$, the IPM indexed by a \emph{symmetric} function space $\mathcal{F}$ is defined as follows \cite{muller1997integral}:
\vskip -0.35  in
\begin{equation}
d_{\mathcal{F}}(\mathbb{P},\mathbb{Q})= \sup_{f \in \mathcal{F}}\Big\{  \underset{x\sim \mathbb{P}}{\mathbb{E}} f(x) -\underset{x\sim \mathbb{Q}}{\mathbb{E}}f(x)\Big\}.
\label{eq:IPM}
\end{equation} 
\vskip -0.1 in
It is easy to see that $d_{\mathcal{F}}$ defines a pseudo-metric over $\mathcal{P}(\pazocal{X})$. Note specifically that if $\mathcal{F}$ is not bounded, $\sup_f$ will scale $f$ to be arbitrarily large.
By choosing $\mathcal{F}$ appropriately \cite{IPMemp}, various distances between probability measures can be defined. 

\textbf{First formulation: Rayleigh Quotient.}
In order to define an IPM in the GAN context, \cite{arjovsky2017wasserstein,mroueh2017mcgan} impose the boundedness of the function space via a data independent constraint. This was achieved via restricting the norms of the weights parametrizing the function space to a $\ell_p$ ball. Imposing such a data independent constraint makes the training highly dependent on the constraint hyper-parameters  and restricts the capacity of the learned network, limiting the usability of the learned critic in a  semi-supervised learning task.
Here we take a different angle and  design the IPM to be scaling invariant as a \emph{Rayleigh quotient}. Instead of measuring the discrepancy between means as in Equation \eqref{eq:IPM}, we measure a  \emph{standardized} discrepancy, so that the distance is \textbf{bounded by construction}. Standardizing this discrepancy introduces as we will see a \emph{data dependent} constraint, that controls the growth of the weights of the critic $f$ and ensures the stability of the training while maintaining the capacity of the critic.  Given two distributions $\mathbb{P},\mathbb{Q} \in \mathcal{P}(\pazocal{X})$ the Fisher IPM for a function space $\mathcal{F}$ is defined as follows:
\begin{equation}
d_{\mathcal{F}}(\mathbb{P},\mathbb{Q})=\sup_{f \in \mathcal{F} }\frac{ \underset{x\sim \mathbb{P}}{\mathbb{E}}[ f(x)] - \underset{x\sim \mathbb{Q}}{\mathbb{E}}[f(x)]}{ \sqrt{\nicefrac{1}{2}\mathbb{E}_{x\sim \mathbb{P} }f^2(x)+\nicefrac{1}{2}\mathbb{E}_{x \sim \mathbb{Q}} f^2(x)} }.
\label{eq:FisherIPM}
\end{equation}

While a standard IPM (Equation \eqref{eq:IPM}) maximizes the discrepancy between the means of a function under two  different distributions, Fisher IPM looks for critic $f$ that achieves a tradeoff between maximizing  the discrepancy between the means under the two distributions (between \emph{class} variance), and reducing the pooled second order moment (an upper bound on the intra-class variance). 

Standardized discrepancies have a long history in statistics and the so-called \emph{two-samples} hypothesis testing. 
For example the classic two samples Student's $t-$ test defines the student statistics as the ratio between means discrepancy and the sum of standard deviations.
It is now well established that learning generative models has its roots in the two-samples hypothesis testing problem \cite{mohamed2016learning}. Non parametric two samples testing and model criticism from the kernel literature lead to the so called maximum kernel mean discrepancy (MMD) \cite{gretton2012kernel}. The MMD cost function and the mean matching IPM for a general function space has been recently used for training GAN \cite{li2015generative,dziugaite2015training,mroueh2017mcgan}. 

Interestingly Harchaoui et al \cite{harchaoui2008testing_nips} proposed Kernel Fisher Discriminant Analysis  for the two samples hypothesis testing problem, and showed its statistical consistency. The Standard Fisher discrepancy used in  Linear Discriminant Analysis (LDA) or Kernel Fisher Discriminant Analysis (KFDA) can be written:
$\sup_{f \in \mathcal{F} }\frac{ \left( \underset{x\sim \mathbb{P}}{\mathbb{E}}[ f(x)] - \underset{x\sim \mathbb{Q}}{\mathbb{E}}[f(x)]\right)^2}{ \rm{Var}_{x\sim\mathbb{P}}(f(x))+\rm{Var}_{x\sim\mathbb{Q}}(f(x)) },$
where $\rm{Var}_{x\sim\mathbb{P}}(f(x))=\mathbb{E}_{x\sim \mathbb{P}}f^2(x)-(\mathbb{E}_{x\sim \mathbb{P}}(f(x)))^2$.
Note that in LDA $\mathcal{F}$ is restricted to linear functions, in KFDA $\mathcal{F}$ is restricted to a Reproducing Kernel Hilbert Space (RKHS).
Our Fisher IPM  (Eq \eqref{eq:FisherIPM}) deviates from  the standard Fisher discrepancy  since the numerator is not squared, and we use in the denominator the second order moments instead of the variances. Moreover in our definition of Fisher IPM, $\mathcal{F}$ can be any symmetric function class.

\textbf{Second formulation: Constrained form.}
Since the distance is scaling invariant, $d_{\mathcal{F}}$ can be written equivalently in the following constrained form:
\begin{equation}
d_{\mathcal{F}}(\mathbb{P},\mathbb{Q})=\sup_{f \in \mathcal{F}, \frac{1}{2}\mathbb{E}_{x\sim \mathbb{P} }f^2(x)+\frac{1}{2}\mathbb{E}_{x \sim \mathbb{Q}} f^2(x)=1 } \mathcal{E}(f):=\underset{x\sim \mathbb{P}}{\mathbb{E}}[ f(x)] - \underset{x\sim \mathbb{Q}}{\mathbb{E}}[f(x)].
\label{eq:FisherIPMConstraint}
\end{equation}

\vskip -0.10 in
\textbf{Specifying $\P,\Q$: Learning GAN with Fisher IPM.} We turn now to  the problem of learning GAN with Fisher IPM.  Given a distribution $\mathbb{P}_r\in \mathcal{P}(\pazocal{X})$, we learn a function $g_{\theta}: \pazocal{Z}\subset \mathbb{R}^{n_z}\to \pazocal{X},$
such that for $z\sim p_{z}$, the distribution of $g_{\theta}(z)$ is close to the real data distribution $\mathbb{P}_{r}$, where $p_{z}$ is a fixed distribution on $\pazocal{Z}$ (for instance $z\sim \mathcal{N}(0,I_{n_z})$). 
Let $\mathbb{P}_{\theta}$ be the distribution of $g_{\theta}(z),z\sim p_{z}$. Using Fisher IPM (Equation \eqref{eq:FisherIPMConstraint}) indexed by a parametric function class $\mathcal{F}_{p}$,
the generator minimizes the IPM:
$\min_{g_{\theta}} d_{\mathcal{F}_{p}}(\mathbb{P}_r,\mathbb{P}_{\theta})$.
Given samples $\{x_i ,1\dots N\}$ from $\mathbb{P}_r$ and samples $\{z_i ,1\dots M\}$ from $p_{z}$ we shall solve the following empirical problem:
\vskip -0.25 in
\begin{align}
&\min_{g_{\theta}} \sup_{f_{p}\in \mathcal{F}_{p}} \hat{\mathcal{E}}(f_p,g_{\theta}):=\frac{1}{N} \sum_{i=1}^N f_{p}(x_i) - \frac{1}{M}\sum_{j=1}^M f_{p}(g_{\theta}(z_j))\text{ Subject to } \hat{\Omega}(f_p,g_{\theta})=1,
\label{eq:empFisherGAN}
\end{align}
\vskip -0.15 in
where $\hat{\Omega}(f_p,g_{\theta})= \frac{1}{2 N} \sum_{i=1}^N f^2_{p}(x_i) + \frac{1}{2 M}\sum_{j=1}^M f^2_{p}(g_{\theta}(z_j))$.
For simplicity we will have $M=N$.

\begin{figure}[t]
\centering
\includegraphics[width=0.75\linewidth]{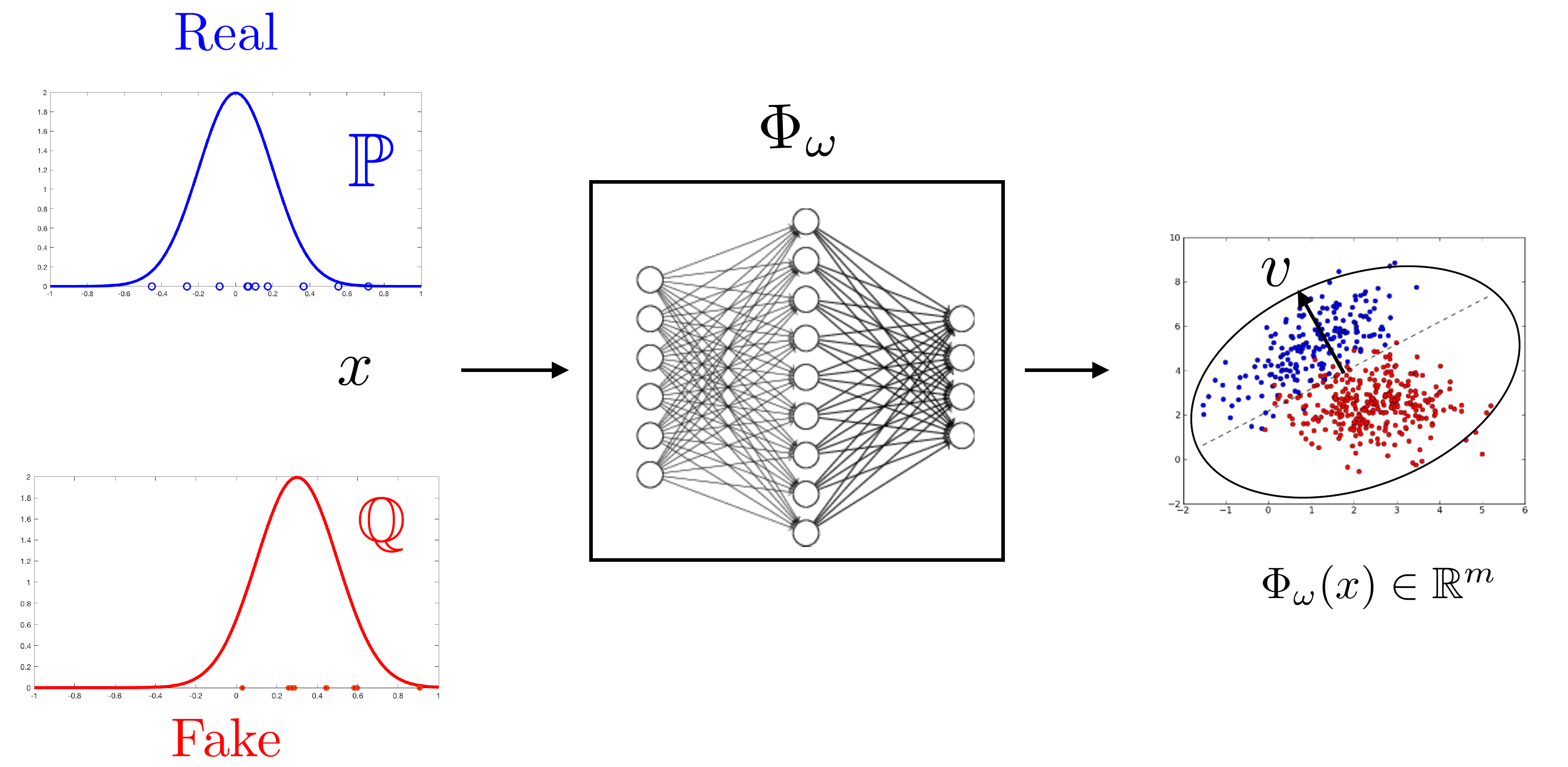}
\caption{Illustration of Fisher IPM with Neural Networks. $\Phi_\omega$ is a convolutional neural network which defines the embedding space.
$v$ is the direction in this embedding space with maximal mean separation
$\scalT{v}{\mu_{\omega}(\mathbb{P})-\mu_{\omega}(\mathbb{Q})}$,
constrained by the hyperellipsoid
$v^\top \, \Sigma_{\omega}(\mathbb{P};\mathbb{Q}) \, v = 1$.
}
\label{fig:illustration}
\vskip -0.1 in
\end{figure}

\subsection{Fisher IPM with Neural Networks}\label{sec:whiteFisher}
We will specifically study the case where $\mathcal{F}$ is a finite dimensional Hilbert space induced by a neural network $\Phi_\omega$
(see Figure \ref{fig:illustration} for an illustration).
In this case, an IPM with data-independent constraint will be equivalent to mean matching \cite{mroueh2017mcgan}.
We will now show that Fisher IPM will give rise to a \emph{whitened} mean matching interpretation, or equivalently to \emph{mean matching} with a Mahalanobis distance.

\textbf{Rayleigh Quotient.} Consider the function space $\mathcal{F}_{v,\omega}$, defined as follows
$$\mathcal{F}_{v,\omega}=\{f(x)=\scalT{v}{\Phi_{\omega}(x)} | v \in \mathbb{R}^m, \Phi_{\omega}: \pazocal{X} \to \mathbb{R}^m \},$$
$\Phi_{\omega}$ is typically parametrized with a multi-layer neural network.
We define the mean and covariance (Gramian) feature embedding of a distribution as in McGan \cite{mroueh2017mcgan}:
$$\mu_{\omega}(\mathbb{P})=\underset{x \sim \mathbb{P}}{\mathbb{E}}\left( \Phi_{\omega}(x)\right) ~~~~~\text{and}~~~~~ \Sigma_{\omega}(\mathbb{P})=\underset{x\sim \mathbb{P}}{\mathbb{E}}\left(\Phi_{\omega}(x)\Phi_{\omega}(x)^{\top}\right),$$  Fisher IPM as defined in Equation \eqref{eq:FisherIPM} on $\mathcal{F}_{v,\omega}$ can be  written as follows:
\begin{equation}
d_{\mathcal{F}_{v,\omega}}(\mathbb{P},\mathbb{Q})=\max_{\omega} \max_{v} \frac{\scalT{v}{\mu_{\omega}(\mathbb{P})-\mu_{\omega}(\mathbb{Q})}}{\sqrt{v^{\top}(\frac{1}{2}\Sigma_{\omega}(\mathbb{P})+ \frac{1}{2}\Sigma_{\omega}(\mathbb{Q})+\gamma I_m )v}},
\end{equation}
where we added a regularization term ($\gamma>0$) to avoid singularity of the covariances.
Note that if  $\Phi_{\omega}$ was implemented with homogeneous non linearities  such as RELU, if we swap $(v,\omega)$ with $(c v, c'\omega)$ for any constants $c,c'>0$, the distance $d_{\mathcal{F}_{v,\omega}}$ remains unchanged, hence the scaling invariance.

\textbf {Constrained Form.} Since the Rayleigh Quotient is not amenable to optimization, we will consider Fisher IPM as a constrained optimization problem.
By virtue of the scaling invariance and the constrained form of the Fisher IPM given in  Equation \eqref{eq:FisherIPMConstraint}, $d_{\mathcal{F}_{v,\omega}}$ can be written equivalently as:
\begin{equation}
d_{\mathcal{F}_{v,\omega}}(\mathbb{P},\mathbb{Q})=\max_{\omega,v,v^{\top}(\frac{1}{2}\Sigma_{\omega}(\mathbb{P})+\frac{1}{2} \Sigma_{\omega}(\mathbb{Q})+\gamma I_m )v=1} \scalT{v}{\mu_{\omega}(\mathbb{P})-\mu_{\omega}(\mathbb{Q})}
\end{equation}
Define the pooled covariance: $\Sigma_{\omega}(\mathbb{P};\mathbb{Q})=\frac{1}{2}\Sigma_{\omega}(\mathbb{P})+\frac{1}{2} \Sigma_{\omega}(\mathbb{Q})+\gamma I_m$.
Doing a simple change of variable $u=(\Sigma_{\omega}(\mathbb{P};\mathbb{Q}) )^{\frac{1}{2}} v$ we see that:
\begin{eqnarray}
d_{\mathcal{F}_{u,\omega}}(\mathbb{P},\mathbb{Q})
&=& \max_{\omega}\max_{u,\nor{u}=1 }  \scalT{u}{(\Sigma_{\omega}(\mathbb{P};\mathbb{Q}))^{-\frac{1}{2}} \left(\mu_{\omega}(\mathbb{P})- \mu_{\omega}(\mathbb{Q})\right)}\nonumber\\
&=& \max_{\omega} \nor{(\Sigma_{\omega}(\mathbb{P};\mathbb{Q}) )^{-\frac{1}{2}} \left(\mu_{\omega}(\mathbb{P})- \mu_{\omega}(\mathbb{Q})\right)},
\end{eqnarray}
hence we see that fisher IPM corresponds to the worst case distance between \emph{whitened means}.
Since the means are white, we don't need to impose further constraints on $\omega$ as in \cite{arjovsky2017wasserstein,mroueh2017mcgan}.
Another interpretation of the Fisher IPM stems from the fact that:
$$d_{\mathcal{F}_{v,\omega}}(\mathbb{P},\mathbb{Q})=\max_{\omega} \sqrt{(\mu_{\omega}(\mathbb{P})-\mu_{\omega}(\mathbb{Q}))^{\top}\Sigma_{\omega}^{-1}(\mathbb{P};\mathbb{Q})(\mu_{\omega}(\mathbb{P})-\mu_{\omega}(\mathbb{Q}))},$$
from which we see that Fisher IPM is a  Mahalanobis distance  between the mean feature embeddings of the distributions. The Mahalanobis distance is defined by the positive definite matrix $\Sigma_{w}(\mathbb{P};\mathbb{Q})$.
We show in Appendix \ref{appendix:wgangpvsfisher} that the gradient penalty in Improved Wasserstein \cite{gulrajani2017improved} gives rise to a similar  Mahalanobis mean matching interpretation.

\textbf{Learning GAN with Fisher IPM.} Hence we see that learning GAN with Fisher IPM:
$$\min_{g_{\theta}} \max_{\omega} \max_{v, v^{\top}(\frac{1}{2}\Sigma_{\omega}(\mathbb{P}_r)+\frac{1}{2}\Sigma_{\omega}(\mathbb{\mathbb{P}_{\theta}})+\gamma I_{m} )v=1} \scalT{v}{\mu_{w}(\mathbb{P}_r)-\mu_{\omega}(\mathbb{P}_{\theta})}$$
corresponds to a min-max game between a feature space and a generator. The feature space tries to maximize the Mahalanobis distance between the feature means embeddings of \emph{real} and \emph{fake} distributions. The generator tries to minimize the mean embedding distance.

\begin{figure}[t]
\vskip -0.10 in
\centering
\includegraphics[width=\linewidth]{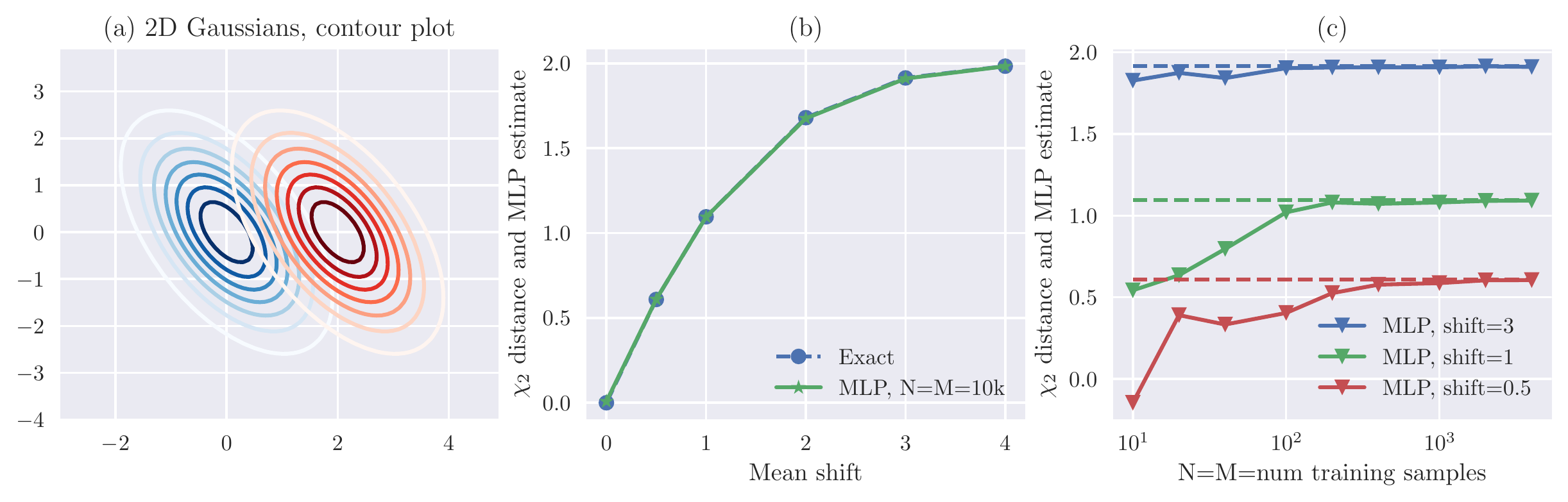}
\vskip -0.10 in
\caption{
Example on 2D synthetic data, where both $\textcolor{blue}{\mathbb{P}}$ and $\textcolor{red}{\mathbb{Q}}$ are fixed normal distributions with the same covariance and shifted means along the x-axis, see (a).
Fig (b, c) show the exact $\chi_2$ distance from numerically integrating Eq (\ref{eq:chi2}),
together with the estimate obtained from training a 5-layer MLP with layer size = 16 and LeakyReLU nonlinearity
on different training sample sizes.
The MLP is trained using Algorithm 1, where sampling from the generator is replaced by sampling from $\Q$,
and the $\chi_2$ MLP estimate is computed with Equation (\ref{eq:FisherIPM}) on a large number of samples (i.e. out of sample estimate).
We see in (b) that for large enough sample size, the MLP estimate is extremely good.
In (c) we see that for smaller sample sizes, the MLP approximation bounds the ground truth $\chi_2$ from below (see Theorem 2)
and converges to the ground truth roughly as $\mathcal{O}(\frac{1}{\sqrt{N}})$ (Theorem 3).
We notice that when the distributions have small $\chi_2$ distance, a larger training size is needed to get a better estimate - again this is in line with Theorem 3.
}
\label{fig:toydata}
\end{figure}

\vskip -0.1in
\section{Theory}\label{sec:theory}
\vskip -0.1in
We will start first by studying the Fisher IPM  defined in Equation \eqref{eq:FisherIPM} when the function space has \emph{full capacity} i.e when the critic belongs to $\mathcal{L}_2(\pazocal{X},\frac{1}{2}(\mathbb{P}+\mathbb{Q}))$ meaning that $\int_{\pazocal{X}} f^2(x)\frac{(\mathbb{P}(x)+\mathbb{Q}(x))}{2}dx <\infty$.
Theorem \ref{theo:chisquarefullcapacity} shows that under this condition, the Fisher IPM corresponds to the \emph{Chi-squared distance} between distributions, and gives a closed form expression of the \emph{optimal critic} function $f_{\chi}$  (See Appendix \ref{appendix:chisq_pearson} for its relation with the Pearson Divergence).
Proofs are given in Appendix \ref{appendix:proofs}.
\begin{theorem} [Chi-squared distance at full capacity] Consider the Fisher IPM for $\mathcal{F}$ being the space of all measurable functions endowed by $\frac{1}{2}(\mathbb{P}+\mathbb{Q})$, i.e. $\mathcal{F}:=\mathcal{L}_{2}(\pazocal{X},\frac{\mathbb{P}+\mathbb{Q}}{2})$.
Define the Chi-squared distance between two distributions:
\begin{equation}
\chi_{2}(\mathbb{P},\mathbb{Q})=\sqrt{\int_{\pazocal{X}} \frac{(\mathbb{P}(x)-\mathbb{Q}(x))^2}{\frac{\mathbb{P}(x)+\mathbb{Q}(x)}{2}} dx}
\label{eq:chi2}
\end{equation}
\vskip -0.1in
The following holds true for any $\mathbb{P},\mathbb{Q}$, $\mathbb{P}\neq \mathbb{Q}$:\\

1) The Fisher IPM for $\mathcal{F}=\mathcal{L}_{2}(\pazocal{X},\frac{\mathbb{P}+\mathbb{Q}}{2})$ is equal to the Chi-squared distance defined above:
$d_{\mathcal{F}}(\mathbb{P},\mathbb{Q})= \chi_2(\mathbb{P},\mathbb{Q}).$\\
2) The optimal critic of the Fisher IPM on $\mathcal{L}_{2}(\pazocal{X},\frac{\mathbb{P}+\mathbb{Q}}{2})$ is :
$$f_{\chi}(x)=\frac{1}{\chi_2(\mathbb{P},\mathbb{Q})} \frac{\mathbb{P}(x)-\mathbb{Q}(x)}{\frac{\mathbb{P}(x)+\mathbb{Q}(x)}{2}}.$$ 
\label{theo:chisquarefullcapacity}
\end{theorem}
\vskip -0.14 in
We note here that LSGAN \cite{mao2016least} at full capacity corresponds to a Chi-Squared divergence,
with the main difference that LSGAN has different objectives for the generator and the discriminator (bilevel optimizaton),
and hence does not optimize a single objective that is a  distance between distributions.
The Chi-squared divergence can also be achieved in the $f$-gan framework from \cite{nowozin2016f}.
We discuss the advantages of the Fisher formulation in Appendix \ref{appendix:fgan}.

Optimizing over $\mathcal{L}_{2}(\pazocal{X},\frac{\mathbb{P}+\mathbb{Q}}{2})$ is not tractable, hence we have to restrict our function class, to a hypothesis class $\mathcal{H}$, that enables tractable computations. Here are some typical choices of the space $\mathcal{H}:$ Linear functions in the input features, RKHS, a non linear multilayer neural network with a linear last layer ($\mathcal{F}_{v,\omega}$). In this Section we don't make any assumptions about the function space and show in Theorem \ref{theo:ChiSquareapproxinH} how the Chi-squared distance is approximated in $\mathcal{H}$, and how this depends on the approximation error of the optimal critic $f_{\chi}$  in $\mathcal{H}$.
\vskip -0.1 in
\begin{theorem}[Approximating Chi-squared distance in an arbitrary function space $\mathcal{H}$]
Let $\mathcal{H}$ be an arbitrary symmetric function space. 
We define the inner product $\scalT{f}{f_{\chi}}_{\mathcal{L}_2(\pazocal{X},\frac{\mathbb{P}+\mathbb{Q}}{2})}=\int_{\pazocal{X}}f(x)f_{\chi}(x)\frac{\mathbb{P}(x)+\mathbb{Q}(x)}{2} dx$,
which induces the Lebesgue norm.
Let $\mathbb{S}_{\mathcal{L}_2(\pazocal{X},\frac{\mathbb{P}+\mathbb{Q}}{2})}$ be the unit sphere in $\mathcal{L}_2(\pazocal{X},\frac{\mathbb{P}+\mathbb{Q}}{2})$:
$\mathbb{S}_{\mathcal{L}_2(\pazocal{X},\frac{\mathbb{P}+\mathbb{Q}}{2})}=\{f:\pazocal{X}\to \mathbb{R},  \nor{f}_{\mathcal{L}_2(\pazocal{X},\frac{\mathbb{P}+\mathbb{Q}}{2})}=1 \}.$
 The fisher IPM defined on an arbitrary function space $\mathcal{H}$ $d_{\mathcal{H}}(\mathbb{P},\mathbb{Q})$, approximates the Chi-squared distance. The approximation quality depends on the cosine of the approximation of the optimal critic $f_{\chi}$ in $\mathcal{H}$. Since $\mathcal{H}$ is symmetric this cosine is always positive (otherwise the same equality holds with an absolute value)
\begin{equation*}
d_{\mathcal{H}}(\mathbb{P},\mathbb{Q})=  \chi_{2}(\mathbb{P},\mathbb{Q})  \sup_{f \in \mathcal{H} \cap ~ \mathbb{S}_{\mathcal{L}_2(\pazocal{X},\frac{\mathbb{P}+\mathbb{Q}}{2})}}\scalT{f}{f_{\chi}}_{\mathcal{L}_2(\pazocal{X},\frac{\mathbb{P}+\mathbb{Q}}{2})},
\end{equation*} 
\vskip -0.15 in
 Equivalently we have following relative approximation error:\\
$$
\frac{ \chi_2(\mathbb{P},\mathbb{Q})- d_{\mathcal{H}}(\mathbb{P},\mathbb{Q})}{\chi_2(\mathbb{P},\mathbb{Q})} =\frac{1}{2} \inf_{f\in \mathcal{H}\cap ~ \mathbb{S}_{\mathcal{L}_2(\pazocal{X},\frac{\mathbb{P}+\mathbb{Q}}{2})}} \nor{f-f_{\chi}}^2_{\mathcal{L}_2(\pazocal{X},\frac{\mathbb{P}+\mathbb{Q}}{2})}.
$$

\label{theo:ChiSquareapproxinH} 
\end{theorem}

From Theorem \ref{theo:ChiSquareapproxinH}, we know that we have always $d_{\mathcal{H}}(\mathbb{P},\mathbb{Q})\leq \chi_{2}(\mathbb{P},\mathbb{Q})$. Moreover if the space $\mathcal{H}$ was rich enough to provide a good approximation of the optimal critic $f_{\chi}$, then $d_{\mathcal{H}}$ is a good approximation of  the Chi-squared distance $\chi_2$. 

Generalization bounds for the sample quality of the estimated  Fisher IPM from samples from $\mathbb{P}$ and $\mathbb{Q}$ can be done akin to \cite{IPMemp}, with the main difficulty that for Fisher IPM we have to bound  the excess risk of a cost function with data dependent constraints  on the function class. We give generalization bounds for learning the Fisher IPM in the supplementary material (\textbf{Theorem 3}, Appendix \ref{appendix:thm3}). 
In a nutshell the generalization error of the critic learned in a hypothesis  class $\mathcal{H}$ from samples of $\mathbb{P}$ and $\mathbb{Q}$,  decomposes to the approximation error from Theorem \ref{theo:ChiSquareapproxinH} and a statistical error that is bounded using data dependent local Rademacher complexities  \cite{bartlett2005} and scales like $O(\sqrt{\nicefrac{1}{n}}),n=\nicefrac{MN}{M+N}$.
We illustrate in Figure \ref{fig:toydata} our main theoretical claims on a toy problem.

\section{Fisher GAN Algorithm using ALM}
For any choice of the parametric function class $\mathcal{F}_{p}$ (for example $\mathcal{F}_{v,\omega}$), note the constraint in Equation \eqref{eq:empFisherGAN} by
$\hat{\Omega}(f_p,g_{\theta}) = \frac{1}{2 N} \sum_{i=1}^N f^2_{p}(x_i) + \frac{1}{2 N}\sum_{j=1}^N f^2_{p}(g_{\theta}(z_j)).$
Define the Augmented Lagrangian \cite{Nocedal2006NO} corresponding to Fisher GAN  objective and constraint given in Equation \eqref{eq:empFisherGAN}:
\begin{equation}
\pazocal{L}_{F}(p,\theta,\lambda)= \hat{\mathcal{E}}(f_p,g_\theta)+ \lambda(1-\hat{\Omega}(f_p,g_{\theta}))-\frac{\rho}{2}(\hat{\Omega}(f_{p},g_{\theta})-1)^2
\label{eq:alm}
\end{equation}
\vskip -0.15 in
where $\lambda$ is the Lagrange multiplier and $\rho>0$ is the quadratic penalty weight.
We alternate between optimizing the critic and the generator.  Similarly to \cite{gulrajani2017improved} we impose the constraint when training the critic only. Given $\theta$, for training the critic we solve $\max_{p} \min_{\lambda} \pazocal{L}_{F}(p,\theta,\lambda).$
Then given  the critic parameters $p$ we optimize the generator weights  $\theta$ to minimize the objective $ \min_{\theta}\hat{\mathcal{E}}(f_p,g_\theta) .$
We give in Algorithm \ref{alg:FisherGAN}, an algorithm for Fisher GAN, note that we use ADAM \cite{kingma2014adam} for optimizing the parameters of the critic and the generator. We use SGD for the Lagrange multiplier with learning rate $\rho$ following practices in Augmented Lagrangian \cite{Nocedal2006NO}. 
\vskip -0.12in
\begin{algorithm}[ht!]
   \caption{Fisher GAN}
   \label{alg:FisherGAN}
\begin{algorithmic}
   \STATE {\bfseries Input:} $\rho$ penalty weight, $\eta$ Learning rate, $n_c$ number of iterations for training the critic, N batch size
   \STATE {\bfseries Initialize} $p,\theta, \lambda=0$ 
   \REPEAT
   \FOR{$j=1$ {\bfseries to} $n_c$}
   \STATE Sample  a minibatch $x_i,i=1\dots N, x_i \sim \mathbb{P}_r$ 
   \STATE Sample  a minibatch $z_i,i=1\dots N, z_i \sim p_z$ 
   \STATE $(g_{p},g_{\lambda})\gets (\nabla_{p} {\pazocal{L}_{F}},\nabla_{\lambda}\pazocal{L}_{F})(p,\theta,\lambda) $
   \STATE $ p\gets p +\eta \text{ ADAM }(p,g_{p})$\\
   \STATE $\lambda \gets \lambda - \rho g_{\lambda}$ \COMMENT{SGD rule on $\lambda$ with learning rate $\rho$}
   \ENDFOR
   \STATE Sample ${z_i,i=1\dots N , z_i \sim p_z}$
   \STATE $d_{\theta}\gets \nabla_{\theta}\hat{\mathcal{E}}(f_p,g_{\theta})= -\nabla_{\theta}\frac{1}{N}\sum_{i=1}^N f_{p}(g_{\theta}(z_i))$ 
   \STATE $\theta \gets \theta -\eta \text{ ADAM }(\theta,d_{\theta})$
   \UNTIL{$\theta$  converges}
\end{algorithmic}
\end{algorithm}
\section{Experiments}
We experimentally validate the proposed Fisher GAN.
We claim three main results:
(1) stable training with a meaningful and stable loss going down as training progresses and correlating with sample quality, similar to \cite{arjovsky2017wasserstein, gulrajani2017improved}.
(2) very fast convergence to good sample quality as measured by inception score.
(3) competitive semi-supervised learning performance, on par with literature baselines, without requiring normalization of the critic.

We report results on three benchmark datasets: CIFAR-10 \cite{cifar10}, LSUN  \cite{yu2015lsun} and CelebA \cite{liu2015deep}.
We parametrize the generator $g_\theta$ and critic $f$ with convolutional neural networks following the model design from DCGAN \cite{radford2015unsupervised}.
For $64 \times 64$ images (LSUN, CelebA) we use the model architecture in Appendix \ref{sec:mdl1}, 
for CIFAR-10 we train at a $32 \times 32$ resolution using architecture in \ref{sec:mdl2} 
for experiments regarding sample quality (inception score),
while for semi-supervised learning we use a better regularized discriminator similar to the Openai \cite{salimans2016improved} and ALI \cite{dumoulin2016adversarially} architectures, as given in \ref{sec:mdl3}.
We used Adam \cite{kingma2014adam} as optimizer for all our experiments, hyper-parameters given in Appendix \ref{sec:hypers}.

\textbf{Qualitative: Loss stability and sample quality.}
Figure \ref{fig:samples} shows samples and plots during training.
For LSUN we use a higher number of D updates ($n_c=5$) , 
since we see similarly to WGAN that the loss shows large fluctuations with lower $n_c$ values.
For CIFAR-10 and CelebA we use reduced $n_c=2$ with no negative impact on loss stability.
CIFAR-10 here was trained without any label information.
We show both train and validation loss on LSUN and CIFAR-10 showing, as can be expected,
no overfitting on the large LSUN dataset and some overfitting on the small CIFAR-10 dataset.
To back up our claim that Fisher GAN provides stable training,
we trained both a Fisher Gan and WGAN where the batch normalization in the critic $f$ was removed (Figure \ref{fig:nobn}).
\begin{figure}[t]
\centering
  \includegraphics[width=0.3\textwidth]{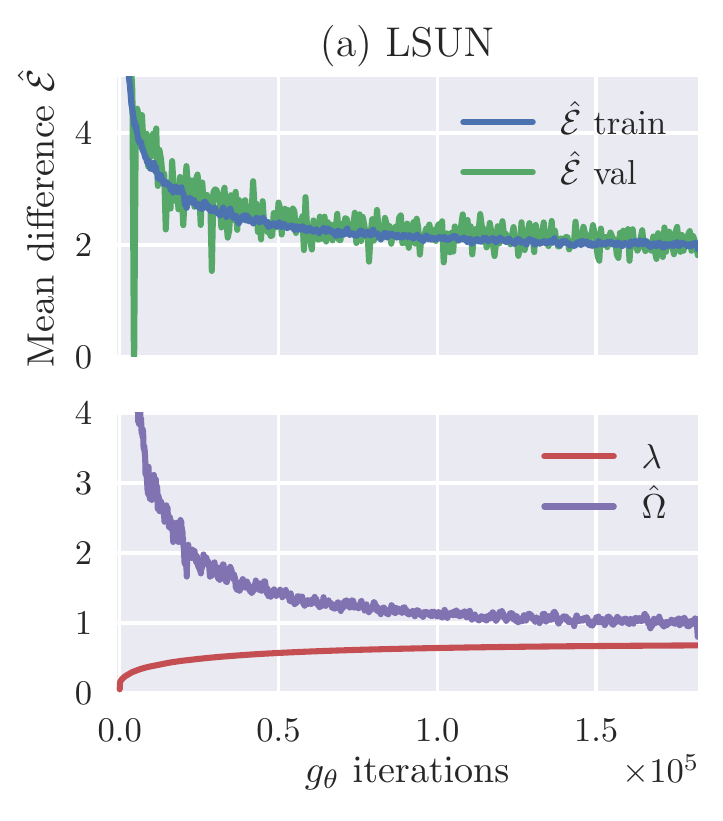}
  \includegraphics[width=0.3\textwidth]{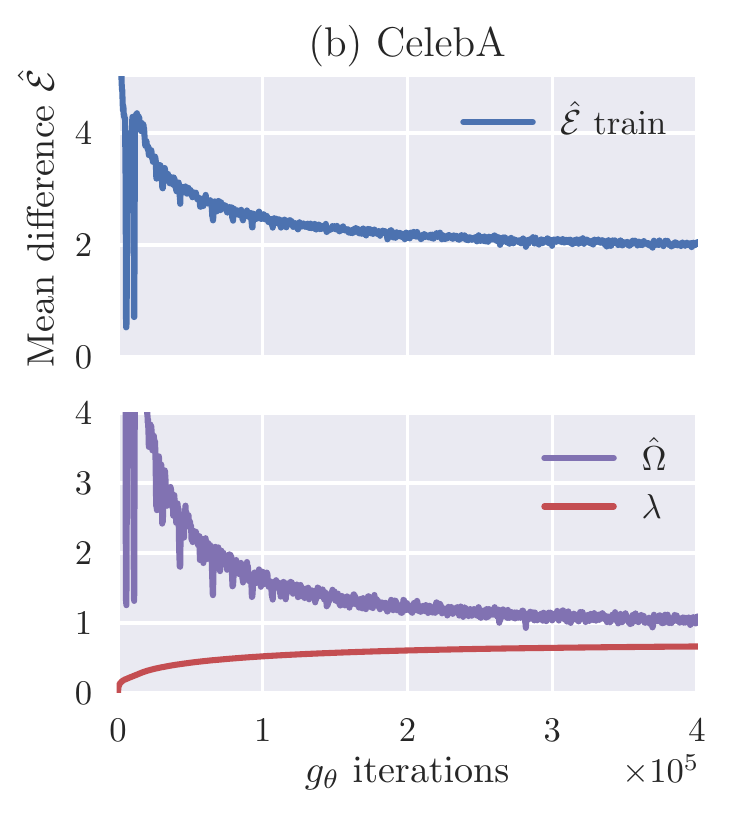}
  \includegraphics[width=0.3\textwidth]{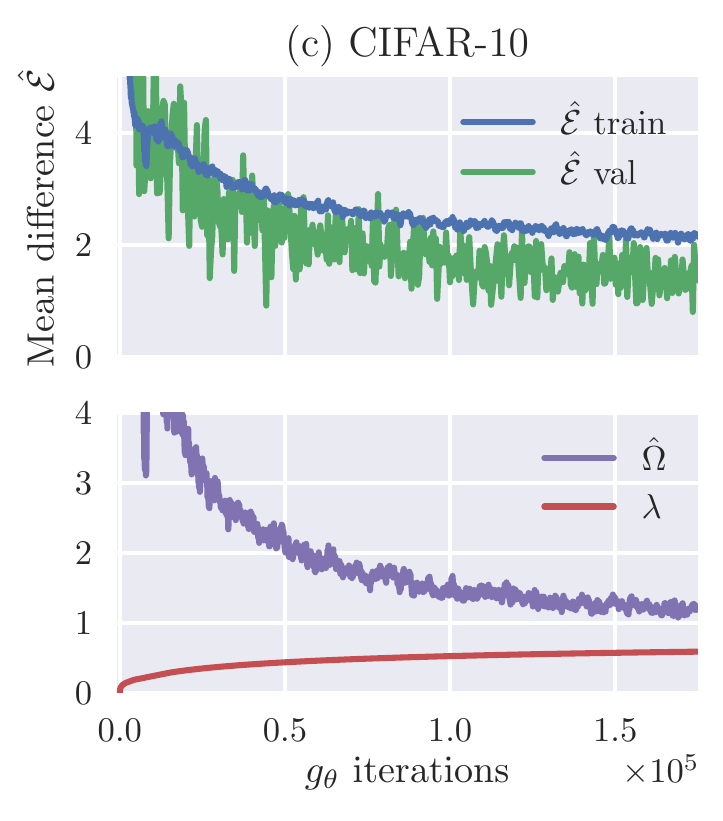}
\\ 
  \includegraphics[width=0.3\textwidth]{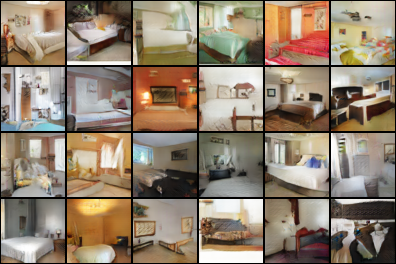}
  \includegraphics[width=0.3\textwidth]{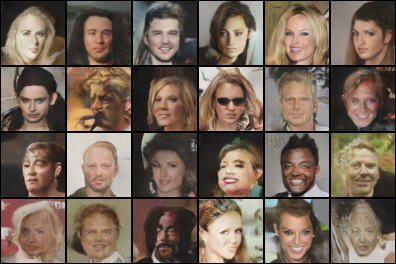}
  \includegraphics[width=0.3\textwidth]{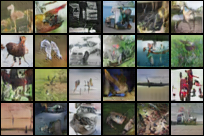}
  \caption{Samples and plots 
  of the loss $\hat{\mathcal{E}}(.)$, lagrange multiplier $\lambda$, and constraint $\hat{\Omega}(.)$
  on 3 benchmark datasets.
  We see that during training as $\lambda$ grows slowly, the constraint becomes tight.
}
  \label{fig:samples}
  \vskip -0.13in
\end{figure}

\begin{figure}[t]
  \centering
  \includegraphics[width=0.3\textwidth]{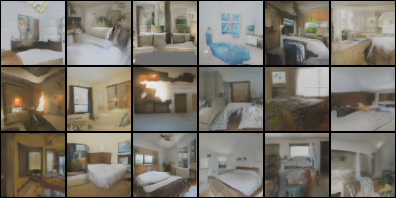}
  \includegraphics[width=0.3\textwidth]{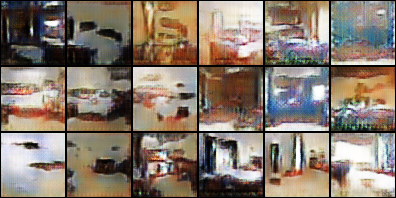}
  \caption{ \textbf{No Batch Norm:} Training results from a critic $f$ without batch normalization.
  Fisher GAN (left) produces decent samples, while WGAN with weight clipping (right) does not.
  We hypothesize that this is due to the implicit whitening that Fisher GAN provides.
  (Note that WGAN-GP does also succesfully converge without BN \cite{gulrajani2017improved}).
  For both models the learning rate was appropriately reduced.
  }
  \label{fig:nobn}
  \vskip -0.15in
\end{figure}

\textbf{Quantitative analysis: Inception Score and Speed.}
It is agreed upon that evaluating generative models is hard \cite{theis2015note}.
We follow the literature in using ``inception score'' \cite{salimans2016improved} as a metric for the quality of CIFAR-10 samples.
Figure \ref{fig:inception} shows the inception score as a function of number of $g_\theta$ updates and wallclock time.
All timings are obtained by running on a single K40 GPU on the same cluster.
We see from Figure \ref{fig:inception}, that Fisher GAN both produces better inception scores, and has a clear speed advantage over WGAN-GP.

\textbf{Quantitative analysis: SSL.} One of the main premises of unsupervised learning, is to learn features on a large corpus of unlabeled data in an unsupervised fashion,
which are then transferable to other tasks.
This provides a proper framework to measure the performance of our algorithm.
This leads us to quantify the performance of Fisher GAN by semi-supervised learning (SSL) experiments on CIFAR-10.
We do joint supervised and unsupervised training on CIFAR-10,
by adding a cross-entropy term to the IPM objective, in conditional and unconditional generation.
\begin{figure}[t]
\centering
  \includegraphics[width=0.60\textwidth,valign=t]{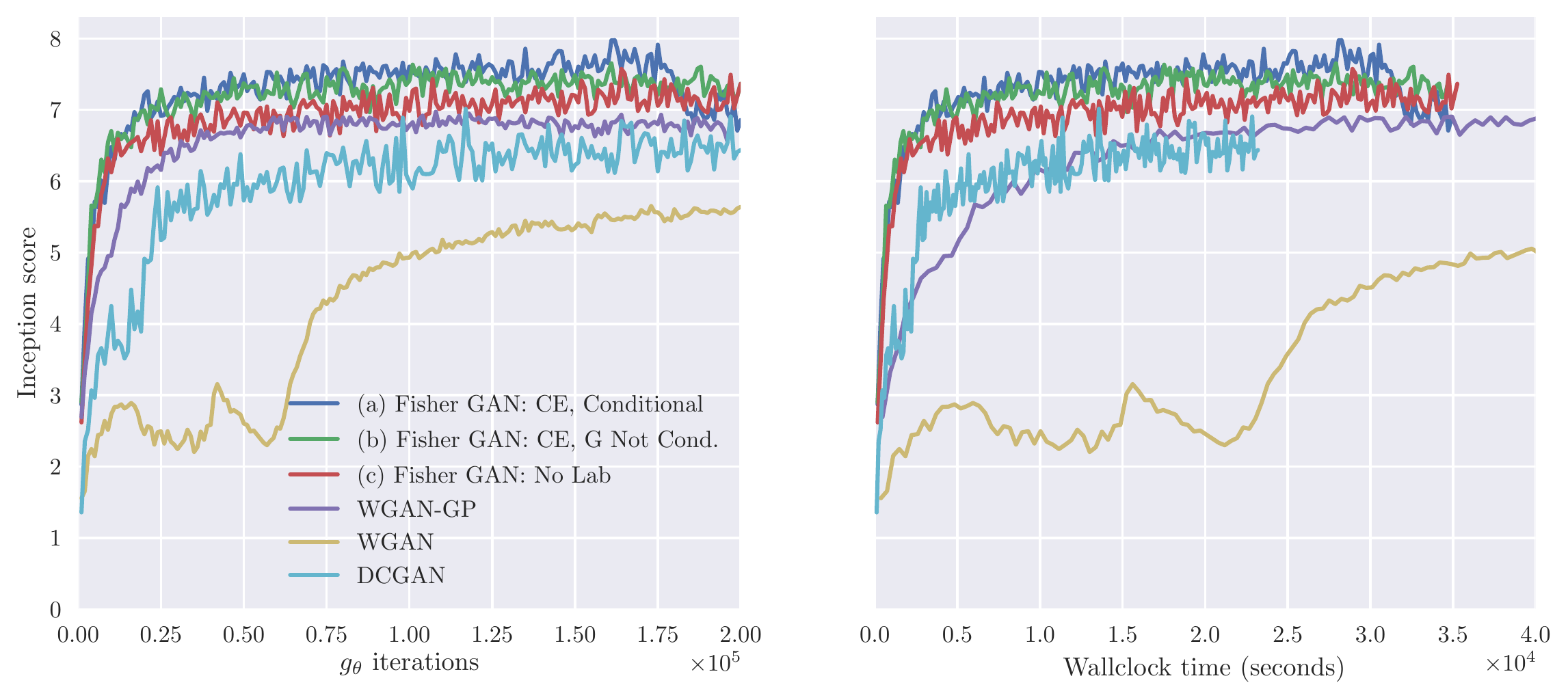}
  \includegraphics[width=0.35\textwidth,valign=t]{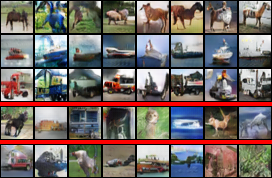}
  \caption{CIFAR-10 inception scores under 3 training conditions.
  Corresponding samples are given in rows from top to bottom (a,b,c).
  The inception score plots are mirroring Figure 3 from \cite{gulrajani2017improved}. \\
  \ifarxiv
  \textbf{Note} In v1 of this paper, the baseline inception scores were underestimated because they were computed using too few samples.\\
   \fi
  \textbf{Note} All inception scores are computed from the same tensorflow codebase, using the architecture described in appendix \ref{sec:mdl2},
  and with weight initialization from a normal distribution with stdev=0.02.
  In Appendix \ref{appendix:inception} we show that these choices are also benefiting our WGAN-GP baseline.
}
  \label{fig:inception}
  \vskip -0.2in
\end{figure}

\begin{table}[ht!]
\vskip -0.13in
\caption{\label{tab:inception_resnet} CIFAR-10 inception scores using resnet architecture and codebase from \cite{gulrajani2017improved}.
We used Layer Normalization \cite{ba2016layer} which outperformed unnormalized resnets.
Apart from this, no additional hyperparameter tuning was done to get stable training of the resnets.
}
\centering
\begin{tabular}{p{1.6in} p{0.6in}} \toprule
  Method & Score \\
\midrule
ALI \cite{dumoulin2016adversarially}                & $5.34 \pm .05$ \\
BEGAN \cite{berthelot2017began}                     & $5.62$ \\
DCGAN \cite{radford2015unsupervised} (in \cite{huang2016stacked})  & $6.16 \pm .07$ \\
Improved GAN (-L+HA) \cite{salimans2016improved}    & $6.86 \pm .06$ \\
EGAN-Ent-VI \cite{dai2017calibrating}               & $7.07 \pm .10$ \\
DFM \cite{warde2017improving}                       & $7.72 \pm .13$ \\
WGAN-GP ResNet \cite{gulrajani2017improved}         & $7.86 \pm .07$ \\
\emph{\textbf{Fisher GAN ResNet (ours)}}            & $7.90 \pm .05$ \\
\bottomrule
\vskip .1em
\qquad \qquad Unsupervised
\end{tabular}
\quad
\begin{tabular}{p{1.65in} p{0.6in}} \toprule
  Method & Score \\
\midrule
SteinGan \cite{wang2016learning}  & $6.35$ \\
DCGAN (with labels, in \cite{wang2016learning}) & $6.58$ \\
Improved GAN \cite{salimans2016improved}  & $8.09 \pm .07$ \\
\emph{Fisher GAN ResNet (ours) }          & $8.16 \pm .12$ \\
AC-GAN \cite{odena2016conditional}        & $8.25 \pm .07$ \\
SGAN-no-joint \cite{huang2016stacked}     & $8.37 \pm .08$ \\
WGAN-GP ResNet \cite{gulrajani2017improved} & $8.42 \pm .10$ \\
\textbf{SGAN \cite{huang2016stacked} }    & $8.59 \pm .12$ \\
\bottomrule
\vskip .1em
\qquad \qquad Supervised
\end{tabular}
\end{table}

\textbf{Unconditional Generation with CE Regularization.}
We parametrize the critic $f$ as in $\mathcal{F}_{v,\omega}$. While training the critic using the Fisher GAN objective $\pazocal{L}_{F}$ given in Equation \eqref{eq:alm}, we train a linear classifier on the feature space $\Phi_{\omega}$ of the critic, whenever labels are available ($K$ labels). The linear classifier is trained with Cross-Entropy (CE) minimization.
Then the critic loss becomes
$\pazocal{L}_{D} = \pazocal{L}_{F} - \lambda_D \sum_{(x,y) \in \text{lab}} CE(x,y;S,\Phi_{\omega})$, where  $CE(x,y; S, \Phi_\omega) = -\log \left[\rm{Softmax}(\scalT{S}{\Phi_{\omega}(x)})_y \right]$,
where $S \in \mathbb{R}^{K\times m}$ is the linear classifier and $\scalT{S}{\Phi_\omega} \in \mathbb{R}^{K}$ with slight abuse of notation.
$\lambda_{D}$ is the regularization hyper-parameter.
We now sample three minibatches for each critic update: one labeled batch from the small labeled dataset for the CE term, and an unlabeled batch + generated batch for the IPM.

\textbf{Conditional Generation with CE Regularization.}
We also trained \emph{conditional generator} models,
conditioning the generator on $y$ by concatenating the input noise with a 1-of-K embedding of the label: we now have $g_\theta(z,y)$.
We parametrize the critic in $\mathcal{F}_{v,\omega}$ and modify the critic objective as above.  We also add a cross-entropy term for the generator to minimize during its training step:
$\pazocal{L}_G = \hat{\mathcal{E}} + \lambda_G \sum_{z \sim p(z), y \sim p(y)} CE(g_\theta(z,y),y;S,\Phi_{\omega})$.
For generator updates we still need to sample only a single minibatch since we use the minibatch of samples from $g_\theta(z,y)$ to compute both the IPM loss $\mathcal{\hat{E}}$ and CE.
The labels are sampled according to the prior $y \sim p(y)$, which defaults to the discrete uniform prior when there is no class imbalance.
We found $\lambda_D=\lambda_G=0.1$ to be optimal.

\textbf{New Parametrization of the Critic: ``$K+1$ SSL''.} One specific successful formulation of SSL in the standard GAN framework was provided in \cite{salimans2016improved},
where the discriminator classifies samples into $K+1$ categories: the $K$ correct clases, and $K+1$ for fake samples. 
Intuitively this puts the real classes in competition with the fake class. 
In order to implement this idea in the Fisher framework, we define a new function class of the critic that puts in competition the $K$ class directions of the classifier $S_y$, and another ``K+1'' direction $v$ that indicates fake samples.
Hence we propose the following parametrization for the critic: $ f(x) = \sum_{y=1}^K p(y|x)\scalT{S_y}{\Phi_{\omega}(x)} - \scalT{v}{\Phi_{\omega}(x)}$, 
where $p(y|x)=\rm{Softmax}(\scalT{S}{\Phi_{\omega}(x)})_y$ which is also optimized with Cross-Entropy.
Note that this critic does not fall under the interpretation with whitened means from Section \ref{sec:whiteFisher},
but does fall under the general Fisher IPM framework from Section \ref{sec:fisheripm}.
We can use this critic with both conditional and unconditional generation in the same way as described above.
In this setting we found $\lambda_D=1.5,\, \lambda_G=0.1$ to be optimal.

\textbf{Layerwise normalization on the critic.}
For most GAN formulations following DCGAN design principles, batch normalization (BN) \cite{ioffe2015batch} in the critic is an essential ingredient.
From our semi-supervised learning experiments however, it appears that batch normalization gives substantially worse performance than layer normalization (LN) \cite{ba2016layer} or even no layerwise normalization.
We attribute this to the implicit whitening Fisher GAN provides.

Table \ref{tab:cifar} shows the SSL results on CIFAR-10.
We show that Fisher GAN has competitive results, on par with state of the art literature baselines.
When comparing to WGAN with weight clipping, it becomes clear that we recover the lost SSL performance.
Results with the $K+1$ critic are better across the board, proving consistently the advantage of our proposed $K+1$ formulation.
Conditional generation does not provide gains in the setting with layer normalization or without normalization.

\begin{table}[ht!]
\vskip -0.13in
\caption{\label{tab:cifar} CIFAR-10 SSL results. \ifarxiv \\ \textbf{Note} In v3, strong results were added using LN and no normalization. \fi
 }
\centering
\resizebox{\textwidth}{!}{\begin{tabular}{@{}lllll@{}} \toprule
Number of labeled examples & 1000 & 2000 & 4000 & 8000 \\
Model & \multicolumn{4}{c}{Misclassification rate} \\ \midrule
CatGAN \cite{springenberg2015unsupervised}    &              &      & $19.58$ &   \\
Improved GAN (FM) \cite{salimans2016improved} & $21.83 \pm 2.01$ & $19.61 \pm 2.09$ & $18.63 \pm 2.32$ & $17.72 \pm 1.82$ \\
ALI \cite{dumoulin2016adversarially}          & $19.98 \pm 0.89$ & $19.09 \pm 0.44$ & $17.99 \pm 1.62$ & $17.05 \pm 1.49$ \\
\midrule
WGAN (weight clipping) Uncond  & 69.01 & 56.48 & 40.85 & 30.56 \\ 
WGAN (weight clipping) Cond    & 68.11 & 58.59 & 42.00 & 30.91 \\
\midrule
\midrule
Fisher GAN BN Cond     &  $36.37$ &  $32.03$ &  $27.42$ &  $22.85$ \\
Fisher GAN BN Uncond   &  $36.42$ &  $33.49$ &  $27.36$ &  $22.82$ \\
Fisher GAN BN K+1 Cond   &  $34.94$ &  $28.04$ &  $23.85$ &  $20.75$ \\
Fisher GAN BN K+1 Uncond &  $33.49$ &  $28.60$ &  $24.19$ &  $21.59$ \\
\midrule
Fisher GAN LN Cond                   &  $26.78 \pm 1.04$ &  $23.30 \pm 0.39$ &  $20.56 \pm 0.64$ &  $18.26 \pm 0.25$ \\
Fisher GAN LN Uncond                 &  $24.39 \pm 1.22$ &  $22.69 \pm 1.27$ &  $19.53 \pm 0.34$ &  $17.84 \pm 0.15$ \\
Fisher GAN LN K+1 Cond                 &  $20.99 \pm 0.66$ &  $19.01 \pm 0.21$ &  $17.41 \pm 0.38$ &  $15.50 \pm 0.41$ \\
Fisher GAN LN K+1, Uncond              &  $19.74 \pm 0.21$ &  $17.87 \pm 0.38$ &  $16.13 \pm 0.53$ &  $14.81 \pm 0.16$ \\
\midrule
Fisher GAN No Norm K+1, Uncond & $21.15 \pm 0.54$ &  $18.21 \pm 0.30$ &  $16.74 \pm 0.19$ &  $14.80 \pm 0.15$ \\ 
\bottomrule
\end{tabular}}
\vskip -0.2in
\end{table}

\section{Conclusion}
We have defined Fisher GAN, which provide a stable and fast way of training GANs.
The Fisher GAN is based on a scale invariant IPM, by constraining the second order moments of the critic.
We provide an interpretation as whitened (Mahalanobis) mean feature matching and $\chi_2$ distance.
We show graceful theoretical and empirical advantages of our proposed Fisher GAN.
\paragraph{Acknowledgments.} The authors thank Steven J. Rennie for many helpful discussions and Martin Arjovsky for helpful clarifications and pointers.

\bibliographystyle{unsrt}
\bibliography{refs,simplex}

\newpage
\begin{center}
\large{\textbf{Supplementary Material  for Fisher GAN}}\\
Youssef Mroueh$^*$, Tom Sercu$^*$\\
IBM Research AI
\end{center}
\appendix

\section{WGAN-GP versus Fisher GAN}
\label{appendix:wgangpvsfisher}
Consider $$\mathcal{F}_{v,\omega}=\{f(x)=\scalT{v}{\Phi_{w}(x)} , v\in \mathbb{R}^m, \Phi_{\omega}: \pazocal{X}\subset \mathbb{R}^d\to \mathbb{R}^{m} \}$$
Let $$J_{\Phi_{\omega}}(x) \in \mathbb{R}^{m\times d}, [J_{\Phi_{\omega}}(x)]_{i,j}= \frac{\partial \scalT{e_i}{\Phi_{\omega}(x) }}{\partial x_j} $$ be the Jacobian matrix of the $\Phi_{\omega}(.)$.
It is easy to see that
$$\nabla_{x}f(x)=J^{\top}_{\Phi_{\omega}}(x)v \in \mathbb{R}^d, $$
and therefore $$ \nor{\nabla_{x}f(x)}^2=\scalT{v}{J_{\Phi_{\omega}}(x)J^{\top}_{\Phi_{\omega}}(x) v},$$
Note that ,
$$J_{\Phi_{\omega}}(x)J^{\top}_{\Phi_{\omega}}(x)$$
is the so called \emph{metric tensor} in information geometry (See for instance \cite{metrictensor} and references there in). 
The gradient penalty for WGAN of \cite{gulrajani2017improved} can be derived from a Rayleigh quotient principle as well, written in the constraint form:
$$d_{\mathcal{F}_{v,\omega}}(\mathbb{P},\mathbb{Q})  = \sup_{f \in \mathcal{F}_{v,\omega}, \mathbb{E}_{u\sim U[0,1]}\mathbb{E}_{x\sim u\mathbb{P}+(1-u)\mathbb{Q}}\nor{\nabla_xf(x)}^2=1} \mathbb{E}_{x\sim \mathbb{P}}f(x)-\mathbb{E}_{x\sim \mathbb{Q}}f(x)  $$
Using the special parametrization we can write: 
$$\mathbb{E}_{u\sim U[0,1]}\mathbb{E}_{x\sim u\mathbb{P}+(1-u)\mathbb{Q}}\nor{\nabla_xf(x)}^2= v^{\top} \left(\mathbb{E}_{u\sim U[0,1]}\mathbb{E}_{x\sim u\mathbb{P}+(1-u)\mathbb{Q}}J_{\Phi_{\omega}}(x)J^{\top}_{\Phi_{\omega}}(x)\right) v $$
Let $$\mathcal{M}_{\omega}(\mathbb{P};\mathbb{Q})= \mathbb{E}_{u\sim U[0,1]}\mathbb{E}_{x\sim u\mathbb{P}+(1-u)\mathbb{Q}}J_{\Phi_{\omega}}(x)J^{\top}_{\Phi_{\omega}}(x)\in \mathbb{R}^{m\times m}$$
is the expected Riemannian  metric tensor \cite{metrictensor}.
Hence we obtain:
\begin{align*}
d_{\mathcal{F}_{v,\omega}}(\mathbb{P},\mathbb{Q}) &=\max_{w}\max_{v, v^{\top}\mathcal{M}_{\omega}(\mathbb{P};\mathbb{Q})v=1}\scalT{v}{\mu_{\omega}(\mathbb{P})-\mu_{\omega}(\mathbb{Q})} \\
&=\max_{\omega} \nor{\mathcal{M}^{-\frac{1}{2}}_{\omega}(\mathbb{P};\mathbb{Q})(\mu_{\omega}(\mathbb{P}))-\mu_{\omega}(\mathbb{Q})}
\end{align*}
Hence Gradient penalty can be seen as well as mean matching in the metric defined by the expected metric tensor $\mathcal{M}_{\omega}$.

Improved WGAN \cite{gulrajani2017improved} IPM can be written as follows : $$\max_{\omega} \sqrt{(\mu_{\omega}(\mathbb{P})-\mu_{\omega}(\mathbb{Q}))^{\top}\mathcal{M}_{\omega}^{-1}(\mathbb{P};\mathbb{Q})(\mu_{\omega}(\mathbb{P})-\mu_{\omega}(\mathbb{Q}))}$$
to be contrasted with Fisher IPM:
$$\max_{\omega} \sqrt{(\mu_{\omega}(\mathbb{P})-\mu_{\omega}(\mathbb{Q}))^{\top}\Sigma_{\omega}^{-1}(\mathbb{P};\mathbb{Q})(\mu_{\omega}(\mathbb{P})-\mu_{\omega}(\mathbb{Q}))}$$

Both Improved WGAN are doing mean matching using different Mahalanobis distances! While improved WGAN uses an \emph{ expected metric tensor} $\mathcal{M}_{\omega}$ to compute this distance, Fisher IPM uses a simple pooled covariance $\Sigma_{\omega}$ to compute this metric. It is clear that Fisher GAN has a computational advantage!

\section{Chi-squared distance and Pearson Divergence }
\label{appendix:chisq_pearson}
The definition of $\chi_2$ distance:
$$\chi_2^2(\mathbb{P},\mathbb{Q})=2\int_{\pazocal{X}} \frac{(\mathbb{P}(x)-\mathbb{Q}(x))^2}{\mathbb{P}(x)+\mathbb{Q}(x)} dx. $$
The $\chi_2$  Pearson divergence:
$$\chi^{P}_2(\mathbb{P},\mathbb{Q})=\int_{\pazocal{X}}\frac{(\mathbb{P}(x)-\mathbb{Q}(x))^2}{\mathbb{Q}(x)} dx.$$
We have the following relation:
$$\chi^2_2(\mathbb{P},\mathbb{Q})=\frac{1}{4}\chi^{P}_2\left(\mathbb{P},\frac{\mathbb{P}+\mathbb{Q}}{2}\right). $$

\section{Fisher GAN and $\varphi$-divergence Based GANs}
\label{appendix:fgan}
Since $f$-gan \cite{nowozin2016f} also introduces a GAN formulation which recovers the Chi-squared divergence,
we compare our approaches.

Let us recall here the definition of  $\varphi$-divergence:
$$d_{\varphi}(\mathbb{P},\mathbb{Q})=\int_{\pazocal{X}} \varphi\left(\frac{\mathbb{P}(x)}{\mathbb{Q}(x)}\right)\mathbb{Q}(x) dx, $$
where $\varphi :\mathbb{R}^+\to \mathbb{R}$ is a convex, lower-semicontinuous function satisfying $\varphi(1)=0$. Let $\varphi^*$ the Fenchel conjugate of $\varphi$:
$$\varphi^*(t)=\sup_{u \in Dom_{\varphi}}{ut-\varphi(u)}$$
As shown in \cite{nowozin2016f} and in \cite{Sriperumbudur2009OnIP}, for any function space $\mathcal{F}$ we get the lower bound:
$$d_{\varphi}(\mathbb{P},\mathbb{Q})\geq \sup_{f \in \mathcal{F}} \mathbb{E}_{x\sim \mathbb{P}}f(x)-\mathbb{E}_{x\sim \mathbb{Q}}\varphi^*(f(x)) , $$
For the particular case $\varphi(t)=(t-1)^2$ and $\varphi^*(t)=\frac{1}{4}t^2+t$ we have the Pearson $\chi_2$ divergence:
$$d_{\varphi}(\mathbb{P},\mathbb{Q})=\int_{\pazocal{X}} \frac{(\mathbb{P}(x)-\mathbb{Q}(x))^2}{\mathbb{Q}(x)}dx=\chi^{P}_2(\mathbb{P},\mathbb{Q}) $$
Hence to optimize the same cost function of Fisher GAN in the $\varphi$-GAN framework we have to consider:
$$\frac{1}{2}\sqrt{\chi^{{P}}_2\left(\mathbb{P},\frac{\mathbb{P}+\mathbb{Q}}{2}\right)},$$
Fisher GAN gives an inequality for the symmetric Chi-squared and the $\varphi$-GAN gives a lower variational bound. i.e compare for $\varphi$-GAN:
\begin{align}
\sup_{f \in \mathcal{F}} \mathbb{E}_{x\sim \mathbb{P} }f(x)-\mathbb{E}_{x\sim \frac{\mathbb{P}+\mathbb{Q}}{2}}\varphi^*(f(x))&=\sup_{f \in \mathcal{F}}\mathbb{E}_{x \sim \mathbb{P}}f(x)  -\mathbb{E}_{x\sim \frac{\mathbb{P}+\mathbb{Q}}{2}}\left(\frac{1}{4}f^2(x)+f(x)\right)\nonumber\\
&= \sup_{f\in \mathcal{F}} \frac{1}{2}\left(\mathbb{E}_{x\sim \mathbb{P}} f(x)- \mathbb{E}_{x\sim \mathbb{Q}}f(x)\right)-\frac{1}{4}\mathbb{E}_{x\sim \frac{\mathbb{P}+\mathbb{Q}}{2}}f^2(x)
\label{eq:chisquaredfgan}
\end{align}
and for Fisher GAN:
\begin{equation}
\sup_{f\in \mathcal{F}, \mathbb{E}_{x\sim \frac{\mathbb{P}+\mathbb{Q}}{2}}f^2(x)=1} \mathbb{E}_{x\sim \mathbb{P}}f(x)-\mathbb{E}_{x\sim \mathbb{Q}}(f(x))  
\label{eq:chisquaredfisher}
\end{equation}
while equivalent at the optimum those two formulations for the symmetric Chi-squared given in Equations \eqref{eq:chisquaredfgan}, and \eqref{eq:chisquaredfisher} have different theoretical and practical properties.
On the theory side:
\begin{enumerate}
\item  While the formulation in  \eqref{eq:chisquaredfgan} is a  $\varphi$ divergence, the formulation given by the Fisher criterium in \eqref{eq:chisquaredfisher} is an IPM with a data dependent constraint. This is a surprising result because $\varphi$-divergences and IPM exhibit different properties and the only known non trivial  $\varphi$ divergence that is also an IPM with data independent function class is the total variation distance \cite{Sriperumbudur2009OnIP}. When we allow the function class to be dependent on the distributions, the symmetric Chi-squared  divergence (in fact general Chi-squared also) can be cast as an IPM! Hence in the context of GAN training we inherit the known stability of IPM based training for GANs. 
\item Theorem  \ref{theo:ChiSquareapproxinH}  for the Fisher criterium gives us an approximation error when we change the function from the space of measurable functions to a hypothesis class. It is not clear how tight the lower bound in the $\varphi$-divergence will be as we relax the function class.
\end{enumerate}
On the practical side:
\begin{enumerate}
\item Once we parametrize the critic $f$ as a neural network with linear output activation, i.e. $f(x)=\scalT{v}{\Phi_{\omega}(x)}$, 
  we see that the optimization is unconstrained for the $\varphi$-divergence formulation \eqref{eq:chisquaredfgan} and the weights updates can explode and have an 
  unstable behavior. On the other hand in the Fisher formulation \eqref{eq:chisquaredfisher} the data dependent constraint that is imposed slowly through the lagrange multiplier, 
  enforces a variance control that prevents the critic from blowing up and causing instabilities in the training. Note that in the Fisher case we have three players: the critic, 
  the generator and the lagrange multiplier. The lagrange multiplier grows slowly to enforce the constraint and to approach the Chi-squared distance as training converges. 
  Note that the $\varphi$-divergence formulation \eqref{eq:chisquaredfgan}
  can be seen as a Fisher GAN with fixed lagrange multiplier $\lambda=\frac{1}{2}$ that is indeed unstable in theory and in our experiments.
\end{enumerate}
\begin{remark} Note that if the Neyman divergence is of interest, it can also be obtained as the following Fisher criterium:
\begin{equation}
\sup_{f\in \mathcal{F}, \mathbb{E}_{x\sim \mathbb{P}}f^2(x)=1} \mathbb{E}_{x\sim \mathbb{P}}f(x)-\mathbb{E}_{x\sim \mathbb{Q}}(f(x)) , 
\end{equation}
this is equivalent at the optimum to:
$$ \chi^N_2(\mathbb{P},\mathbb{Q})=\int_{\pazocal{X}} \frac{(\mathbb{P}(x)-\mathbb{Q}(x))^2}{\mathbb{P}(x)} dx.$$
Using a neural network $f(x)=\scalT{v}{\Phi_{\omega}(x)}$, the Neyman divergence can be achieved with linear output activation and a data dependent constraint:
$$\sup_{v,\omega,~ \mathbb{E}_{x\sim \mathbb{P} } (\scalT{v}{\Phi_{\omega}(x)})^2=1} \scalT{v}{\mathbb{E}_{x\sim \mathbb{P}} \Phi_{\omega}(x)-\mathbb{E}_{x\sim \mathbb{Q}}\Phi_{\omega}(x)}$$

To obtain the same divergence as a $\varphi$-divergence we need $\varphi(u)=\frac{(1-u)^2}{u}$, and $\varphi^*(u)=2-2\sqrt{1-u} ,~(u<1) .$
Moreover exponential activation functions are used in \cite{nowozin2016f}, which most likely renders this formulation also unstable for GAN training.

\end{remark}

\section{Proofs}
\label{appendix:proofs}
\begin{proof}[Proof of Theorem \ref{theo:chisquarefullcapacity}]
Consider the space of measurable functions,  $$\mathcal{F}=\left\{f: \pazocal{X} \to \mathbb{R}, \text{ $f$ measurable such that } \int_{\pazocal{X}} f^2(x)\frac{(\mathbb{P}(x)+\mathbb{Q}(x))}{2}dx <\infty \right \}$$
meaning that $f\in \mathcal{L}_{2}(\pazocal{X},\frac{\mathbb{P}+\mathbb{Q}}{2})$.
\begin{eqnarray*}
d_{\mathcal{F}}(\mathbb{P},\mathbb{Q})&=&\sup_{f \in \mathcal{L}_{2}(\pazocal{X},\frac{\mathbb{P}+\mathbb{Q}}{2}), f\neq 0 }\frac{ \underset{x\sim \mathbb{P}}{\mathbb{E}}[ f(x)] - \underset{x\sim \mathbb{Q}}{\mathbb{E}}[f(x)]}{ \sqrt{\frac{1}{2}\mathbb{E}_{x\sim \mathbb{P} }f^2(x)+\frac{1}{2}\mathbb{E}_{x \sim \mathbb{Q}} f^2(x)} }\\
&=& \sup_{f \in \mathcal{L}_{2}(\pazocal{X},\frac{\mathbb{P}+\mathbb{Q}}{2}), \nor{f}_{\mathcal{L}_2(\pazocal{X},\frac{\mathbb{P}+\mathbb{Q}}{2})}=1 }\underset{x\sim \mathbb{P}}{\mathbb{E}}[ f(x)] - \underset{x\sim \mathbb{Q}}{\mathbb{E}}[f(x)]\\
&=& \sup_{f \in \mathcal{L}_{2}(\pazocal{X},\frac{\mathbb{P}+\mathbb{Q}}{2}), \nor{f}_{\mathcal{L}_2(\pazocal{X},\frac{\mathbb{P}+\mathbb{Q}}{2})}\leq1 }\underset{x\sim \mathbb{P}}{\mathbb{E}}[ f(x)] - \underset{x\sim \mathbb{Q}}{\mathbb{E}}[f(x)] \text{ (By convexity of the cost functional in  $f$)}\\
&=& \sup_{f \in \mathcal{L}_{2}(\pazocal{X},\frac{\mathbb{P}+\mathbb{Q}}{2}) }\inf _{\lambda \geq 0} \pazocal{L}(f,\lambda),
\end{eqnarray*}
where in the last equation we wrote  the lagrangian  of the Fisher IPM for this particular function class $\mathcal{F}:=\mathcal{L}_{2}(\pazocal{X},\frac{\mathbb{P}+\mathbb{Q}}{2})$:
$$\pazocal{L}(f,\lambda)= \int_{\pazocal{X}} f(x) (\mathbb{P}(x)-\mathbb{Q}(x)) dx + \frac{\lambda}{2}\left(1- \frac{1}{2}\ \int_{\pazocal{X}} f^2(x)(\mathbb{P}(x)+\mathbb{Q}(x))dx\right),$$
 By convexity of the functional cost and constraints, and since $f\in \mathcal{L}_{2}(\pazocal{X},\frac{\mathbb{P}+\mathbb{Q}}{2}) $, we can minimize the inner loss to optimize this functional for each $x\in \pazocal{X}$ \cite{convexbook}. The first order conditions of optimality (KKT conditions) gives us for the optimum $f_{\chi},\lambda_*$:
$$(\mathbb{P}(x)-\mathbb{Q}(x)) -\frac{\lambda_*}{2}f_{\chi}(x) (\mathbb{P}(x)+\mathbb{Q}(x))) =0 ,$$
$$f_{\chi}(x)= \frac{2}{\lambda_*} \frac{\mathbb{P}(x)-\mathbb{Q}(x)}{\mathbb{P}(x)+\mathbb{Q}(x)}.$$

Using the feasibility constraint: 
$ \int_{\pazocal{X}} f^2_{\chi}(x)\left(\frac{\mathbb{P}(x)+\mathbb{Q}(x)}{2}\right) =1, $
we get :
$$\int_{\pazocal{X}} \frac{4}{\lambda^2_*} \frac{(\mathbb{P}(x)-\mathbb{Q}(x))^2}{(\mathbb{P}(x)+\mathbb{Q}(x))^2}\left(\frac{\mathbb{P}(x)+\mathbb{Q}(x)}{2}\right) =1, $$
which gives us the expression of $\lambda_*$:
$$\lambda_*=\sqrt{\int_{\pazocal{X}} \frac{(\mathbb{P}(x)-\mathbb{Q}(x))^2}{\frac{\mathbb{P}(x)+\mathbb{Q}(x)}{2}} dx}. $$
Hence for $\mathcal{F}:=\mathcal{L}_{2}(\pazocal{X},\frac{\mathbb{P}+\mathbb{Q}}{2})$ we have:
$${d}_{\mathcal{F}}(\mathbb{P},\mathbb{Q})=  \int_{\pazocal{X}} f_{\chi}(x) (\mathbb{P}(x)-\mathbb{Q}(x)) dx=\sqrt{ \int_{\pazocal{X}} \frac{(\mathbb{P}(x)-\mathbb{Q}(x))^2}{\frac{\mathbb{P}(x)+\mathbb{Q}(x)}{2}} dx}=\lambda_*$$
Define the following distance between two distributions:
$$\chi_{2}(\mathbb{P},\mathbb{Q})= \nor{\frac{d\mathbb{P}}{\frac{d\mathbb{P}+d\mathbb{Q}}{2}}-\frac{d\mathbb{Q}}{\frac{d\mathbb{P}+d\mathbb{Q}}{2}} }_{\mathcal{L}_2(\pazocal{X},
\frac{\mathbb{P}+\mathbb{Q}}{2})},$$
We refer to this distance as  the $\chi_2$ distance between two distributions.
It is easy to see that :
$$d_{\mathcal{F}}(\mathbb{P},\mathbb{Q})= \chi_2(\mathbb{P},\mathbb{Q})$$
and the optimal critic $f_{\chi}$ has the following expression:
$$ f_{\chi}(x)=\frac{1}{\chi_2(\mathbb{P},\mathbb{Q})} \frac{\mathbb{P}(x)-\mathbb{Q}(x)}{\frac{\mathbb{P}(x)+\mathbb{Q}(x)}{2}}.$$
\end{proof}



\begin{proof}[Proof of Theorem \ref{theo:ChiSquareapproxinH}]
Define the means difference functional $\mathcal{E}$:
$$\mathcal{E}(f; \mathbb{P},\mathbb{Q})= \mathbb{E}_{x\sim \mathbb{P}}f(x) -\mathbb{E}_{x \sim \mathbb{Q}}f(x) $$
Let $$\mathbb{S}_{\mathcal{L}_2(\pazocal{X},\frac{\mathbb{P}+\mathbb{Q}}{2})}=\{f:\pazocal{X}\to \mathbb{R},  \nor{f}_{\mathcal{L}_2(\pazocal{X},\frac{\mathbb{P}+\mathbb{Q}}{2})}=1 \}$$
For a symmetric function class $\mathcal{H}$, the Fisher IPM has the following expression:
\begin{eqnarray*}
d_{\mathcal{H}}(\mathbb{P},\mathbb{Q})&=& \sup_{f\in \mathcal{H}, ~ \nor{f}_{\mathcal{L}_2(\pazocal{X},\frac{\mathbb{P}+\mathbb{Q}}{2})}=1  }\mathcal{E}(f;\mathbb{P},\mathbb{Q})\\
&=& \sup_{f \in \mathcal{H} \cap ~ \mathbb{S}_{\mathcal{L}_2(\pazocal{X},\frac{\mathbb{P}+\mathbb{Q}}{2})} }\mathcal{E}(f;\mathbb{P},\mathbb{Q}).
\end{eqnarray*}
Recall that for $\mathcal{H}=\mathcal{L}_2(\pazocal{X}, \frac{\mathbb{P}+\mathbb{Q}}{2})$, the optimum $\chi_{2}(\mathbb{P},\mathbb{Q})$ is achieved for :
$$ f_{\chi}(x)=\frac{1}{\chi_2(\mathbb{P},\mathbb{Q})} \frac{\mathbb{P}(x)-\mathbb{Q}(x)}{\frac{\mathbb{P}(x)+\mathbb{Q}(x)}{2}}, \forall x \in \pazocal{X} \text{ a.s}.$$
Let $f\in \mathcal{H}$ such that  $\nor{f}_{\mathcal{L}_2(\pazocal{X},\frac{\mathbb{P}+\mathbb{Q}}{2})}=1$ we have the following:
\begin{eqnarray*}
\scalT{f}{f_{\chi}}_{\mathcal{L}_2(\pazocal{X},\frac{\mathbb{P}+\mathbb{Q}}{2})}&=&\int_{\pazocal{X}}f(x)f_{\chi}(x)\frac{(\mathbb{P}(x)+\mathbb{Q}(x))}{2}dx\\
&=& \frac{1}{\chi_2(\mathbb{P},\mathbb{Q})} \int_{\pazocal{X}} f(x)(\mathbb{P}(x)-\mathbb{Q}(x)) dx\\
&=& \frac{\mathcal{E}(f;\mathbb{P},\mathbb{Q})}{\chi_2(\mathbb{P},\mathbb{Q})}. 
\end{eqnarray*}

It follows that for any $f \in \mathcal{H} \cap ~ \mathbb{S}_{\mathcal{L}_2(\pazocal{X},\frac{\mathbb{P}+\mathbb{Q}}{2})}$ we have:
\begin{equation}
\mathcal{E}(f;\mathbb{P},\mathbb{Q})= \chi_{2}(\mathbb{P},\mathbb{Q}) \scalT{f}{f_{\chi}}_{\mathcal{L}_2(\pazocal{X},\frac{\mathbb{P}+\mathbb{Q}}{2})}
\end{equation}
In particular taking the $\sup$ over $\mathcal{H} \cap ~ \mathbb{S}_{\mathcal{L}_2(\pazocal{X},\frac{\mathbb{P}+\mathbb{Q}}{2})}$ we have:
\begin{equation}
d_{\mathcal{H}}(\mathbb{P},\mathbb{Q})=  \chi_{2}(\mathbb{P},\mathbb{Q})  \sup_{f \in \mathcal{H} \cap ~ \mathbb{S}_{\mathcal{L}_2(\pazocal{X},\frac{\mathbb{P}+\mathbb{Q}}{2})}}\scalT{f}{f_{\chi}}_{\mathcal{L}_2(\pazocal{X},\frac{\mathbb{P}+\mathbb{Q}}{2})}.
\end{equation} 
note that since $\mathcal{H}$ is symmetric all quantities are positive after taking the sup (if $\mathcal{H}$ was not symmetric one can take the absolute values, and similar results hold with absolute values.)

If $\mathcal{H}$ is rich enough so that we find, for $\varepsilon \in (0,1)$, a  $1-\varepsilon$ approximation of $f_{\chi}$ in $\mathcal{H} \cap ~ \mathbb{S}_{\mathcal{L}_2(\pazocal{X},\frac{\mathbb{P}+\mathbb{Q}}{2})}$, i.e:
$$ \sup_{f \in \mathcal{H} \cap ~ \mathbb{S}_{\mathcal{L}_2(\pazocal{X},\frac{\mathbb{P}+\mathbb{Q}}{2})}}\scalT{f}{f_{\chi}}_{\mathcal{L}_2(\pazocal{X},\frac{\mathbb{P}+\mathbb{Q}}{2})}=1-\varepsilon$$
we have therefore that $d_{\mathcal{H}}$ is a $1-\varepsilon$ approximation of $\chi_{2}(\mathbb{P},\mathbb{Q})$:
$$ d_{\mathcal{H}}(\mathbb{P},\mathbb{Q})= (1-\varepsilon) \chi_{2}(\mathbb{P},\mathbb{Q}).$$
Since $f$ and $f_{\chi}$ are unit norm  in  $\mathcal{L}_{2}(\pazocal{X},\frac{\mathbb{P}+\mathbb{Q}}{2})$ we have the following relative error:
\begin{equation}
\frac{ \chi_2(\mathbb{P},\mathbb{Q})- d_{\mathcal{H}}(\mathbb{P},\mathbb{Q})}{\chi_2(\mathbb{P},\mathbb{Q})} =\frac{1}{2} \inf_{f\in \mathcal{H}\cap ~ \mathbb{S}_{\mathcal{L}_2(\pazocal{X},\frac{\mathbb{P}+\mathbb{Q}}{2})}} \nor{f-f_{\chi}}^2_{\mathcal{L}_2(\pazocal{X},\frac{\mathbb{P}+\mathbb{Q})}{2}}.
\end{equation}

\end{proof}

\section{Theorem 3:	Generalization Bounds}
\label{appendix:thm3}
Let $\mathcal{H}$ be a function space of real valued functions on $\pazocal{X}$. We assume that $\mathcal{H}$ is bounded, there exists $\nu>0$, such that $\nor{f}_{\infty}\leq \nu$.
Since the second moments are bounded we can relax this assumption using Chebyshev's  inequality, we have:
$$\mathbb{P}\left\{x \in \pazocal{X},\abs{f(x)}\leq \nu \right\}\leq \frac{ \underset{x\sim \frac{\mathbb{P}+\mathbb{Q}}{2}}{\mathbb{E}}f^2(x)}{\nu^2}=\frac{1}{\nu^2},$$
hence we have boundedness with high probability.
Define the expected mean discrepancy $\mathcal{E}(.)$ and the second order norm $\Omega(.)$:
$$\mathcal{E}(f)= \mathbb{E}_{x\sim \mathbb{P}}f(x) -\mathbb{E}_{x \sim \mathbb{Q}}f(x)~~, \Omega(f)=\frac{1}{2}\left(\mathbb{E}_{x\sim \mathbb{P}}f^2(x)+ \mathbb{E}_{x\sim \mathbb{Q}}f^2(x)\right )$$
and their empirical counterparts, given $N$ samples $\{x_i\}^N_{i=1}\sim \mathbb{P}$,$\{y_i\}^M_{i=1}\sim \mathbb{Q}$ :
$$\mathcal{\hat{E}}(f)= \frac{1}{N}\sum_{i=1}^N f(x_i) -\frac{1}{M}\sum_{i=1}^M f(y_i),~ \hat{\Omega}(f)= \frac{1}{2N}\sum_{i=1}^N f^2(x_i) +\frac{1}{2M}\sum_{i=1}^M f^2(y_i),$$
\begin{theorem}
Let $n=\frac{MN}{M+N}$. Let $\mathbb{P},\mathbb{Q}\in \mathcal{P}(\pazocal{X}), \mathbb{P}\neq \mathbb{Q}$, and let $\chi_2(\mathbb{P},\mathbb{Q})$ be their Chi-squared distance. 
Let $f^*\in \arg\max_{f\in \mathcal{H},\Omega(f)=1} \mathcal{E}(f)$, and $\hat{f}\in \arg\max_{f\in \mathcal{H},\hat{\Omega}(f)=1} \hat{\mathcal{E}}(f) $.
Define the expected mean discrepancy of the optimal empirical critic $\hat{f}$:
$$\hat{d}_{\mathcal{H}}(\mathbb{P},\mathbb{Q})=\mathcal{E}(\hat{f})$$
For $\tau>0$. The following generalization bound on the estimation of the Chi-squared distance, with probability $1-12e^{-\tau}$:
\begin{equation}
\frac{\chi_2(\mathbb{P},\mathbb{Q})-\hat{d}_{\mathcal{H}}(\mathbb{P},\mathbb{Q})}{\chi_2(\mathbb{P},\mathbb{Q})}\leq \underbrace{\frac{1}{2} \inf_{f\in \mathcal{H}\cap ~ \mathbb{S}_{\mathcal{L}_2(\pazocal{X},\frac{\mathbb{P}+\mathbb{Q}}{2})}} \nor{f-f_{\chi}}^2_{\mathcal{L}_2(\pazocal{X},\frac{\mathbb{P}+\mathbb{Q}}{2})}}_{\text{ approximation error }}+\underbrace{\frac{\varepsilon_n }{\chi_2(\mathbb{P},\mathbb{Q})}}_{\text{Statistical Error}}.
\end{equation}
where 
\begin{align*}
\varepsilon_n&= c_3  \mathcal{R}_{M,N}(f;\{f\in\mathcal{H}, \hat{\Omega}(f) \leq 1+\nu^2+2\eta_n\},S)\\
&+ c_4(1+2\nu \hat{\lambda}) \mathcal{R}_{M,N}(f;\{f \in \mathcal{H},  \hat{\Omega}(f)\leq 1+\frac{\nu^2}{2}+\eta_n\},S) 
+O(\frac{1}{\sqrt{n}})
\end{align*}
and 
$$\eta_n \geq c_1 \nu {\mathcal{R}_{N,M}(f; f\in \mathcal{H},S) }+c_2\frac{\nu^2\tau}{n},$$


 $\hat{\lambda}$ is the Lagrange multiplier, $c_1,c_2,c_3,c_4$ are numerical constants, 
and $\mathcal{R}_{M,N}$  is  the rademacher complexity:
$$\mathcal{R}_{M,N}(f;\mathcal{F},S)= E_{\sigma}\sup_{f \in \mathcal{F}} \left[ \sum_{i=1}^{N+M} \sigma_i \tilde{Y}_if(X_i)|S\right],$$
$ \tilde{Y}=(\underbrace{\frac{1}{N},\dots \frac{1}{N}}_{N},\underbrace{\frac{-1}{M}\dots \frac{-1}{M}}_{M})$,  $S=\{x_1\dots x_N, y_1\dots y_{M} \}$,$\sigma_i=\pm 1$ with probability $\frac{1}{2}$, that are iids.

\label{theo:GenBounds}
\end{theorem}
For example:
$$\mathcal{H}=\{f(x)=\scalT{v}{\Phi(x)}, v\in \mathbb{R}^{m} \}$$ Note that for simplicity here we assume that the feature map is fixed $\Phi: \pazocal{X}\to \mathbb{R}^m$, and we parametrize the class function only with $v$.
 $$\mathcal{R}_{M,N}(f;\{\mathcal{H}, \hat{\Omega}(f) \leq R\},S)) \leq \sqrt{2R \frac{ d(\gamma)}{n}}, $$
where
 $$d(\gamma)=\sum_{j=1}^m \frac{\sigma^2_j}{\sigma^2_j+\gamma}$$
is the effective dimension $(d(\gamma)<<m)$. Hence we see that typically $\varepsilon_n= O(\frac{1}{\sqrt{n}})$.

\begin{proof}[Proof of Theorem \ref{theo:GenBounds}]
Let $\{x_i\}_{i=1}^N \sim \mathbb{P}$, $\{y_i\}_{i=1}^M \sim \mathbb{Q}$.
Define the following functionals:
$$\mathcal{E}(f)= \mathbb{E}_{x\sim \mathbb{P}}f(x) -\mathbb{E}_{x \sim \mathbb{Q}}f(x)~~, \Omega(f)=\frac{1}{2}\left(\mathbb{E}_{x\sim \mathbb{P}}f^2(x)+ \mathbb{E}_{x\sim \mathbb{Q}}f^2(x)\right )$$
and their empirical estimates:
$$\mathcal{\hat{E}}(f)= \frac{1}{N}\sum_{i=1}^N f(x_i) -\frac{1}{M}\sum_{i=1}^M f(y_i),~ \hat{\Omega}(f)= \frac{1}{2N}\sum_{i=1}^N f^2(x_i) +\frac{1}{2M}\sum_{i=1}^M f^2(y_i)$$
Define the following Lagrangians:
$$\pazocal{L}(f,\lambda)=\mathcal{E}(f)+\frac{\lambda}{2}(1-\Omega(f)), ~\pazocal{\hat{L}}(f,\lambda)=\mathcal{\hat{E}}(f)+\frac{\lambda}{2}(1-\hat{\Omega}(f))$$

Recall some definitions of the Fisher IPM:
$$d_{\mathcal{H}}(\mathbb{P},\mathbb{Q})=\sup_{f\in \mathcal{H}}\inf_{\lambda\geq 0}\pazocal{L}(f,\lambda) ~ \text{ achieved at } (f_*,\lambda_*) $$
We assume that a saddle point for this problem exists and it is feasible. We assume also that $\hat{\lambda}$ is positive and bounded.
$$d_{\mathcal{H}}(\mathbb{P},\mathbb{Q})=\mathcal{E}(f_*) \text{ and } \Omega(f_*)=1$$
$$\pazocal{L}(f,\lambda_*)\leq \pazocal{L}(f_*,\lambda_*)\leq  \pazocal{L}(f_*,\lambda)$$
The fisher IPM  empirical estimate is given by:
$$d_{\mathcal{H}}(\mathbb{P}_{N},\mathbb{Q}_{N})=\sup_{f\in \mathcal{H}}\inf_{\lambda\geq 0}\pazocal{\hat{L}}(f,\lambda),~ \text{ achieved at } (\hat{f},\hat{\lambda}) $$
hence we have: 
$$d_{\mathcal{H}}(\mathbb{P}_{N},\mathbb{Q}_{N})=\mathcal{\hat{E}}(\hat{f}) \text{ and } \hat{\Omega}(\hat{f})=1.$$
The Generalization error of the empirical critic $\hat{f}$ is the expected mean discrepancy $\mathcal{E}(\hat{f})$. We note $ \hat{d}_{\mathcal{H}}(\mathbb{P},\mathbb{Q})=\mathcal{E}(\hat{f})$, the estimated distance using the critic $\hat{f}$, on out of samples: 
\begin{align*}
\chi_2(\mathbb{P},\mathbb{Q})-\hat{d}_{\mathcal{H}}(\mathbb{P},\mathbb{Q})&=\mathcal{E}(f_{\chi})-\mathcal{E}(\hat{f})\\
&=\underbrace{\mathcal{E}(f_{\chi})-\mathcal{E}(f^*)}_{\text{Approximation Error}}+\underbrace{\mathcal{E}(f^*)-\mathcal{E}(\hat{f})}_{\text{Statistical Error}}
\end{align*}

\textbf{Bounding the Approximation Error.} By Theorem \ref{theo:ChiSquareapproxinH} we know that:
$$\mathcal{E}(f_{\chi})-\mathcal{E}(f^*)= \chi_2(\mathbb{P},\mathbb{Q})- d_{\mathcal{H}}(\mathbb{P},\mathbb{Q}) =\frac{\chi_2(\mathbb{P},\mathbb{Q})}{2} \inf_{f\in \mathcal{H}\cap ~ \mathbb{S}_{\mathcal{L}_2(\pazocal{X},\frac{\mathbb{P}+\mathbb{Q}}{2})}} \nor{f-f_{\chi}}^2_{\mathcal{L}_2(\pazocal{X},\frac{\mathbb{P}+\mathbb{Q}}{2})}.$$
Hence we have for $\mathbb{P} \neq \mathbb{Q}$:
\begin{equation}
\frac{\chi_2(\mathbb{P},\mathbb{Q})-\hat{d}_{\mathcal{H}}(\mathbb{P},\mathbb{Q})}{\chi_2(\mathbb{P},\mathbb{Q})}=\frac{1}{2} \inf_{f\in \mathcal{H}\cap ~ \mathbb{S}_{\mathcal{L}_2(\pazocal{X},\frac{\mathbb{P}+\mathbb{Q}}{2})}} \nor{f-f_{\chi}}^2_{\mathcal{L}_2(\pazocal{X},\frac{\mathbb{P}+\mathbb{Q}}{2})}+\underbrace{\frac{\mathcal{E}(f^*)-\mathcal{E}(\hat{f})}{\chi_2(\mathbb{P},\mathbb{Q})}}_{\text{Statistical Error}}
\end{equation}
Note that this equation tells us that the relative error depends on the approximation error of the the optimal critic $f_{\chi}$, and the statistical error coming from using finite samples in approximating the distance. We note that the statistical error is divided by the Chi-squared distance, meaning that we need a bigger sample size when $\mathbb{P}$
 and $\mathbb{Q}$ are close in the Chi-squared sense, in order to reduce the overall relative error.
 
Hence we are left with bounding the statistical error using empirical processes theory.
Assume $\mathcal{H}$ is a space of bounded functions i.e $\nor{f}_{\infty}\leq \nu$.

\textbf{Bounding the Statistical Error.}
Note that we have: (i) $\hat{\pazocal{L}}(f^*,\hat{\lambda})\leq \hat{\pazocal{L}}(\hat{f},\hat{\lambda}) $ and (ii) $\Omega(f^*)=1$.
\begin{align*}
\mathcal{E}(f^*)-\mathcal{E}(\hat{f})&= \left(\mathcal{E}(f^*)- \mathcal{\hat{E}}(f^*) \right)+(\underbrace{\mathcal{\hat{E}}(f^*)+\frac{\hat{\lambda}}{2}(1-\hat{\Omega}(f^*))}_{\hat{\pazocal{L}}(f^*,\hat{\lambda})}- \underbrace{\mathcal{\hat{E}}(\hat{f})}_{\hat{\pazocal{L}}(\hat{f},\hat{\lambda})} ) +\left( \mathcal{\hat{E}}(\hat{f})-\mathcal{E}(\hat{f})\right)+\frac{\hat{\lambda}}{2}\left(\hat{\Omega}(f^*)-1\right)\\
&\leq  \sup_{f \in \mathcal{H}, \Omega(f)\leq 1} \abs{\mathcal{\hat{E}}({f})-\mathcal{E}({f}) }+ \sup_{f \in \mathcal{H}, \hat{\Omega}(f)\leq 1} \abs{\mathcal{\hat{E}}({f})-\mathcal{E}({f}) }+ \frac{\hat{\lambda}}{2}\left(\hat{\Omega}(f^*)-\Omega(f^*)\right) \text{Using (i) and (ii)}\\
&\leq  \sup_{f \in \mathcal{H}, \Omega(f)\leq 1} \abs{\mathcal{\hat{E}}({f})-\mathcal{E}({f}) }+ \sup_{f \in \mathcal{H}, \hat{\Omega}(f)\leq 1} \abs{\mathcal{\hat{E}}({f})-\mathcal{E}({f}) }+ \frac{\hat{\lambda}}{2}\sup_{f\in \mathcal{H},\Omega(f)\leq 1}\abs{\hat{\Omega}(f)-\Omega(f)}.
\end{align*}

Let $S=\{x_1\dots x_N, y_1\dots y_{M} \}$. Define the following quantities: 
$$Z_1(S) = \sup_{f \in \mathcal{H},\Omega(f)\leq 1} \abs{\mathcal{\hat{E}}({f})-\mathcal{E}({f}) }, \text{ Concentration of the cost on data distribution dependent constraint}$$
$$ Z_2(S)= \sup_{f \in \mathcal{H},\hat{\Omega}(f)\leq 1} \abs{\mathcal{\hat{E}}({f})-\mathcal{E}({f}) },\text{ Concentration of the cost on an empirical data dependent constraint}  $$
$$ Z_3(S)=\sup_{f\in \mathcal{H},\Omega(f)\leq 1}\abs{\hat{\Omega}(f)-\Omega(f)}, \hat{\lambda} Z_3(S) \text{ is the sensitivity of the cost as the constraint set changes }  $$
We have:
\begin{equation}
\mathcal{E}(f^*)-\mathcal{E}(\hat{f})\leq {Z_1(S)}+{Z_2(S)}+\hat{\lambda} Z_3(S),
\end{equation}

Note that the $\sup$ in   $Z_1(S)$ and  $Z_3(S)$ is taken with respect to class function $\{f, \Omega(f)=\nor{f}^2_{\mathcal{L}_{2}(\pazocal{X},\frac{\mathbb{P}+\mathbb{Q}}{2})}\leq 1\}$
hence we will bound $Z_1(S),$ and $Z_3(S)$ using local Rademacher complexity. In  $Z_2(S)$ the $\sup$ is taken on a data dependent function class and can be bounded with local rademacher complexity as well but needs more careful work.

\textbf{Bounding $Z_1(S)$, and $Z_3(S)$}

\begin{lemma} [Bounds with (Local) Rademacher Complexity \cite{IPMemp,bartlett2005}]
Let $Z(S)= \sup_{f\in \mathcal{F}} \mathcal{E}(f)-\hat{\mathcal{E}}(f)$, Assume that $\nor{f}_{\infty}\leq\nu$, for all $f\in \mathcal{F}$.
\begin{itemize}
\item For any $\alpha,\tau  >0$. Define variances $var_{\mathbb{P}}(f)$, and similarly $var_{\mathbb{Q}}(f)$.  Assume $\max(var_{\mathbb{P}}(f),var_{\mathbb{Q}}(f))\leq r$ for any $f\in \mathcal{F}$. We have with probability $1-e^{-\tau}$ : 
$$Z(S)\leq  {(1+\alpha)} E_{S}Z(S)+ \sqrt{\frac{2r\tau(M+N)}{MN}}+ \frac{2 \tau \nu(M+N)}{MN}\left(\frac{2}{3}+\frac{1}{\alpha}\right) $$
The same result holds for : 
$ Z(S)= \sup_{f\in \mathcal{F}} \hat{\mathcal{E}}(f)-\mathcal{E}(f).$
\item By symmetrization we have:
 $\mathbb{E}_{S} Z(S)\leq 2 E_{S} \mathcal{R}_{M,N}(f;\mathcal{F},S) $
where $\mathcal{R}_{M,N}$  is  the rademacher complexity:
$$\mathcal{R}_{M,N}(f;\mathcal{F},S)= E_{\sigma}\sup_{f \in \mathcal{F}} \left[ \sum_{i=1}^{N+M} \sigma_i \tilde{Y}_if(X_i)|S\right],$$
$ \tilde{Y}=(\underbrace{\frac{1}{N},\dots \frac{1}{N}}_{N},\underbrace{\frac{-1}{M}\dots \frac{-1}{M}}_{M})$, $\sigma_i=\pm 1$ with probability $\frac{1}{2}$, that are iids.
\item  We have with probability $1-e^{-\tau}$ for all $\delta \in (0,1)$:
$$E_{S} \mathcal{R}_{M,N}(f;\mathcal{F},S) \leq  \frac{\mathcal{R}_{M,N}(f;\mathcal{F},S)}{1-\delta}+\frac{\tau \nu (M+N)}{MN\delta(1-\delta)}.$$
\end{itemize}
\label{lem:localrad}
\end{lemma}
\begin{lemma}[Contraction Lemma \cite{bartlett2005}] Let $\phi$ be a contraction, that is $\abs{\phi(x)-\phi(y)}\leq L \abs{x-y}$. Then, for every class $\mathcal{F}$,
$$\mathcal{R}_{M,N}(f;\phi \circ  \mathcal{F},S)\leq L \mathcal{R}_{M,N}(f;\mathcal{F},S),$$
$\phi \circ  \mathcal{F}=\{\phi \circ f, f\in \mathcal{F}\}$. 
\end{lemma}
Let $n=\frac{MN}{M+N}$.
Applying Lemma \ref{lem:localrad} for $\mathcal{F}=\{f\in \mathcal{H},\Omega(f)\leq 1\}$.
Since $\Omega(f)\leq 1$, $var_{\mathbb{P}}(f)\leq \Omega(f)\leq1$, and similarly for $var_{\mathbb{Q}}(f)$. Hence $\max(var_{\mathbb{P}}(f),var_{\mathbb{Q}}(f))\leq 1$.
Putting all together we obtain with probability $1-2e^{-\tau}$:
\begin{align}
Z_1(S)&\leq  \frac{2(1+\alpha)}{1-\delta}\mathcal{R}_{M,N}(f;\{f \in \mathcal{H}, \Omega(f) \leq 1\},S) + \sqrt{\frac{2\tau}{n}}
+ \frac{2\tau \nu}{n}\left(\frac{2}{3}+\frac{1}{\alpha}+\frac{1+\alpha}{\delta(1-\delta)}\right)
\label{eq:Z1}
\end{align}
Now tuning to $Z_3(S)$ applying Lemma \ref{lem:localrad} for $\{f^2, f\in \mathcal{H},\Omega(f)\leq 1\}$. Note that $Var(f^2)\leq \mathbb{E}f^4 \leq \Omega(f)\nu^2 \leq \nu^2$. We have that for $\alpha>0$, $\delta \in (0,1)$ and with probability at least $1-2e^{-\tau}$:
\begin{align*}
Z_3(S)&\leq \frac{2(1+\alpha)}{1-\delta}{\mathcal{R}_{N,M}(f^2;\{ f\in \mathcal{H},\Omega(f)\leq 1\},S)}+ \sqrt{\frac{2\tau\nu^2}{n}}+ \frac{2\tau \nu^2}{n}\left(\frac{2}{3}+\frac{1}{\alpha}+ \frac{1+\alpha}{\delta(1-\delta)}\right) 
\end{align*}
Note that applying the contraction Lemma for $\phi(x)=x^2$ (with lipchitz constant $2\nu$ on $[-\nu,\nu]$) we have: 
$$\mathcal{R}_{N,M}(f^2; \{ f\in \mathcal{H},\Omega(f)\leq 1\},S)\leq 2\nu \mathcal{R}_{N,M}(f;\{f\in \mathcal{H},\Omega(f)\leq 1\},S),$$
Hence we have finally:
\begin{align}
Z_3(S)&\leq \frac{4(1+\alpha)\nu}{1-\delta}{\mathcal{R}_{N,M}(f; \{f\in \mathcal{H},\Omega(f)\leq 1\},S)}+ \sqrt{\frac{2\tau\nu^2}{n}}
+ \frac{2\tau \nu^2}{n}\left(\frac{2}{3}+\frac{1}{\alpha}+ \frac{1+\alpha}{\delta(1-\delta)}\right) 
\label{eq:Z3}
\end{align}
Note that the  of complexity of $\mathcal{H}$, depends also upon the distributions $\mathbb{P}$ and $\mathbb{Q}$, since it is defined on the intersection of $\mathcal{H}$ and the unity  ball in $\mathcal{L}_{2}(\pazocal{X},\frac{\mathbb{P}+\mathbb{Q}}{2})$. 

\textbf{From Distributions to Data dependent Bounds.}  
We study how the $\hat{\Omega}(f)$ concentrates uniformly on $\mathcal{H}$.
Note that in this case to apply Lemma \ref{lem:localrad}, we use $r\leq\mathbb{E}(f^4)\leq \nu^4$.  We have with probability $1-2e^{-\tau}$:
$$\hat{\Omega}(f)\leq \Omega(f)+ \frac{4(1+\alpha)\nu}{1-\delta}{\mathcal{R}_{N,M}(f; f\in \mathcal{H},S) }+ \sqrt{\frac{2\tau\nu^4  r}{n}}
+ \frac{2\tau \nu^2}{n}\left(\frac{2}{3}+\frac{1}{\alpha}+ \frac{1+\alpha}{\delta(1-\delta)}\right) 
$$
Now using that for any $\alpha>0$: 
$2\sqrt{uv}\leq \alpha u+\frac{v}{\alpha}$
we have for $\alpha=\frac{1}{2}$:
$\sqrt{\frac{2\tau \nu^4}{n}}\leq \frac{\nu^2}{2} + \frac{4\tau\nu^2}{n}.$
For some universal constants, $c_1,c_2$, let:
$$\eta_n \geq c_1 \nu {\mathcal{R}_{N,M}(f; f\in \mathcal{H},S) }+c_2\frac{\nu^2\tau}{n},$$
we have therefore with probability $1-2e^{-\tau}$:
\begin{equation}
\hat{\Omega}(f)\leq \Omega(f)+\frac{\nu^2}{2}+\eta_n,
\label{eq:unifconcentration1}
\end{equation}
note that Typically $\eta_n=O(\frac{1}{\sqrt{n}})$.\\
The same inequality holds with the same probability:
\begin{equation}
\Omega(f)\leq \hat{\Omega}(f)+\frac{\nu^2}{2}+\eta_n,
\label{eq:unifconcentration2}
\end{equation}
Note that we have now the following inclusion using Equation \eqref{eq:unifconcentration1}: 
$$\{f,f\in \mathcal{H}, \Omega(f) \leq 1 \} \subset \left\{f, f\in \mathcal{H}, \hat{\Omega}(f)\leq 1+\frac{\nu^2}{2}+\eta_n \right  \} $$
Hence:
$$\mathcal{R}_{M,N}(f;\{f\in \mathcal{H}, \Omega(f) \leq 1\},S) \leq \mathcal{R}_{M,N}(f;\{f \in \mathcal{H},  \hat{\Omega}(f)\leq 1+\frac{\nu^2}{2}+\eta_n\},S) $$
Hence we obtain a data dependent bound in Equations \eqref{eq:Z1},\eqref{eq:Z3} with a union bound with probability $1-6e^{-\tau}$.

\textbf{Bounding $Z_2(S)$.}  
Note that concentration inequalities don't apply to $Z_2(S)$ since the cost function and the function class are data dependent. We need to turn the constraint to a data independent constraint i.e does not depend on the training set. 
For $f,\hat{\Omega}(f)\leq 1$, by Equation \eqref{eq:unifconcentration2} we have with probability $1-2e^{-\tau}$:
$$\Omega(f)\leq 1+\frac{\nu^2}{2}+\eta_n,$$
we have therefore the following inclusion with probability $1-2e^{-\tau}$:
$$\{f\in \mathcal{H},\hat{\Omega}(f)\leq 1\}\subset \{f \in \mathcal{H},\Omega(f)\leq 1+\frac{\nu^2}{2}+\eta_n\}$$
Recall that:
$$Z_2(S) =\sup_{f\in \mathcal{H}, \hat{\Omega}(f)\leq 1}  \abs{\mathcal{\hat{E}}({f})-\mathcal{E}({f}) }$$
Hence with probability $1-2e^{-\tau}$:
 $$Z_2(S)\leq \tilde{Z}_2(S)= \sup_{f,f\in \mathcal{H},\Omega(f)\leq 1+\frac{\nu^2}{2}+\eta_n}  \abs{\mathcal{\hat{E}}({f})-\mathcal{E}({f}) } $$

Applying again Lemma \ref{lem:localrad}  on $\tilde{Z}_2(S)$  we have with probability $1-4e^{-\tau}$:
\begin{align*}
Z_2(S)\leq \tilde{Z}_2(S)&\leq  \frac{2(1+\alpha)}{1-\delta}\mathcal{R}_{M,N}(f;\{ f\in \mathcal{H}, \Omega(f) \leq 1+\frac{\nu^2}{2}+\eta_n\},S) + \sqrt{\frac{2\tau(1+\frac{\nu^2}{2}+\eta_n)}{n}}\\
&+ \frac{2\tau \nu}{n}\left(\frac{2}{3}+\frac{1}{\alpha}+\frac{1+\alpha}{\delta(1-\delta)}\right).
\end{align*}
Now reapplying the inclusion using Equation \eqref{eq:unifconcentration1}, we get the following bound on the local rademacher complexity  with probability $1-2e^{-\tau}$:
$$\mathcal{R}_{M,N}(f;\{f \in \mathcal{H}, \Omega(f) \leq 1+\frac{\nu^2}{2}+\eta_n\},S) \leq \mathcal{R}_{M,N}(f;\{f \in \mathcal{H}, \hat{\Omega}(f) \leq 1+\nu^2+2\eta_n\},S)$$
Hence with probability $1-6e^{-\tau}$ we have:
\begin{align*}
Z_2(S)&\leq  \frac{2(1+\alpha)}{1-\delta}\ \mathcal{R}_{M,N}(f;\{f \in \mathcal{H}, \hat{\Omega}(f) \leq 1+\nu^2+2\eta_n\},S) + \sqrt{\frac{2\tau(1+\frac{\nu^2}{2}+\eta_n)}{n}}\\
&+ \frac{2\tau \nu}{n}\left(\frac{2}{3}+\frac{1}{\alpha}+\frac{1+\alpha}{\delta(1-\delta)}\right).
\end{align*}
%

%

\textbf{Putting all together.} We have with probability at least $1-12e^{-\tau}$, for universal constants $c_1,c_2,c_3,c_4$
$$\eta_n \geq c_1 \nu {\mathcal{R}_{N,M}(f; f\in \mathcal{H},S) }+c_2\frac{\nu^2\tau}{n},$$

\begin{align*}
\mathcal{E}(f^*)-\mathcal{E}(\hat{f})&\leq {Z_1(S)}+{Z_2(S)}+\hat{\lambda} Z_3(S)\\
&\leq \varepsilon_n\\
&= c_3  \mathcal{R}_{M,N}(f;\{f\in \mathcal{H}, \hat{\Omega}(f) \leq 1+\nu^2+2\eta_n\},S)\\
&+ c_4(1+2\nu \hat{\lambda}) \mathcal{R}_{M,N}(f;\{f \in \mathcal{H},  \hat{\Omega}(f)\leq 1+\frac{\nu^2}{2}+\eta_n\},S) \\
&+O(\frac{1}{\sqrt{n}}).
\end{align*}
Note that typically $\varepsilon_n= O(\frac{1}{\sqrt{n}})$.
Hence it follows that:
\begin{equation}
\frac{\chi_2(\mathbb{P},\mathbb{Q})-\hat{d}_{\mathcal{H}}(\mathbb{P},\mathbb{Q})}{\chi_2(\mathbb{P},\mathbb{Q})}\leq \underbrace{\frac{1}{2} \inf_{f\in \mathcal{H}\cap ~ \mathbb{S}_{\mathcal{L}_2(\pazocal{X},\frac{\mathbb{P}+\mathbb{Q}}{2})}} \nor{f-f_{\chi}}^2_{\mathcal{L}_2(\pazocal{X},\frac{\mathbb{P}+\mathbb{Q}}{2})}}_{\text{ approximation error }}+\underbrace{\frac{\varepsilon_n }{\chi_2(\mathbb{P},\mathbb{Q})}}_{\text{Statistical Error}}.
\end{equation}
If $\mathbb{P}$ and $\mathbb{Q}$ are close we need more samples to estimate the $\chi_2$ distance and reduce the relative error.

\textbf{Example: Bounding local complexity for a simple linear function class.}
$$\mathcal{H}=\{f(x)=\scalT{v}{\Phi(x)}, v\in \mathbb{R}^{m} \}$$ Note that for simplicity here we assume that the feature map is fixed $\Phi: \pazocal{X}\to \mathbb{R}^m$, and we parametrize the class function only with $v$. Note that $\sup_{v, v^{\top}(\Sigma(\mathbb{P}_{N})+\Sigma(\mathbb{Q}_{M})+\gamma I_m)v\leq 2R }  \scalT{v}{\sum_{i=1}^N \sigma_i\tilde{Y}_i \Phi(X_i)}$
\begin{align*}
&= \sup_{v,\nor{v}\leq 1} \scalT{v }{\left(\frac{\Sigma(\mathbb{P}_{N})+\Sigma(\mathbb{Q}_{M})+\gamma I_m}{2R}\right)^{-\frac{1}{2}}\sum_{i=1}^N \sigma_i\tilde{Y}_i \Phi(X_i) }\\
&= \nor{\left(\frac{\Sigma(\mathbb{P}_{N})+\Sigma(\mathbb{Q}_{M})+\gamma I_m}{2R}\right)^{-\frac{1}{2}}\sum_{i=1}^{N+M} \sigma_i\tilde{Y}_i \Phi(X_i) }\\
&= \sqrt{2R} \sqrt{\sum_{i,j=1}^{N+M} \sigma_i \sigma_j \tilde{Y}_i \tilde{Y}_{j}\Phi(X_i)^{\top}\left({\Sigma(\mathbb{P}_{N})+\Sigma(\mathbb{Q}_{M})+\gamma I_m}\right)^{-1}\Phi(X_j)}
\end{align*}
It follows by Jensen inequality that $  \mathbb{E}_{\sigma} \sup_{v, v^{\top}(\Sigma(\mathbb{P}_{N})+\Sigma(\mathbb{Q}_{M})+\gamma I_m)v \leq \sqrt{2R} } \scalT{v}{\sum_{i=1}^N \sigma_i\tilde{Y}_i \Phi(X_i)}$

\begin{align*}
 &\leq\sqrt{2R} \sqrt{\mathbb{E}_{\sigma} \sum_{i,j=1}^{N+M} \sigma_i \sigma_j \tilde{Y}_i\tilde{Y}_j \Phi(X_i)^{\top}\left({\Sigma(\mathbb{P}_{N})+\Sigma(\mathbb{Q}_{M})}+\gamma I_m\right)^{-1}\Phi(X_j)} \\
 &= \sqrt{2R} \sqrt{\sum_{i=1}^{N+M} \tilde{Y}^2_i \Phi(X_i)^{\top}\left({\Sigma(\mathbb{P}_{N})+\Sigma(\mathbb{Q}_{M})+\gamma I_m}\right)^{-1}\Phi(X_i)}\\
 &= \sqrt{2R}\sqrt{Tr\left(\left(\Sigma(\mathbb{P}_{N})+\Sigma(\mathbb{Q}_{M})+\gamma I_m\right)^{-1} \left(\frac{1}{N} \Sigma(\mathbb{P}_{N})+\frac{1}{M} \Sigma(\mathbb{Q}_{M}) \right)\right)}\\
 &\leq \sqrt{2R \frac{M+N}{MN}} \sqrt{Tr\left(\left(\Sigma(\mathbb{P}_{N})+\Sigma(\mathbb{Q}_{M})+\gamma I_m\right)^{-1} \left( \Sigma(\mathbb{P}_{N})+ \Sigma(\mathbb{Q}_{M}) \right)\right)}
 \end{align*}
 Let $$d(\gamma)= Tr\left(\left(\Sigma(\mathbb{P}_{N})+\Sigma(\mathbb{Q}_{M})+\gamma I_m\right)^{-1} \left( \Sigma(\mathbb{P}_{N})+ \Sigma(\mathbb{Q}_{M}) \right)\right),$$
 $d(\gamma)$ is the so called effective dimension in regression problems.
 Let $\Sigma$ be the singular values of $\Sigma(\mathbb{P}_{N})+\Sigma(\mathbb{Q}_{M})$, 
 $$d(\gamma)=\sum_{j=1}^m \frac{\sigma^2_j}{\sigma^2_j+\gamma}$$
 
 Hence we obtain the following bound on the local rademacher complexity:
 $$\mathcal{R}_{M,N}(f;\{\mathcal{H}, \hat{\Omega}(f) \leq R\},S)) \leq \sqrt{2R \frac{(M+N) d(\gamma)}{MN}}= \sqrt{2R \frac{ d(\gamma)}{n}}$$
 Note that without the local constraint the effective dimension $d(\gamma)$ (typically $d(\gamma)<< m$) is replaced by the ambient dimension $m$.
\end{proof}

\section{Hyper-parameters and Architectures of Discriminator and Generators}
\label{sec:hypers}
For CIFAR-10 we use adam learning rate $\eta=2\mathrm{e}{-4}$, $\beta_1=0.5$ and $\beta_2=0.999$, and penalty weight $\rho=3\mathrm{e}{-7}$,
for LSUN and CelebA we use $\eta=5\mathrm{e}{-4}$, $\beta_1=0.5$ and $\beta_2=0.999$, and $\rho=1\mathrm{e}{-6}$.
We found the optimization to be stable with very similar performance in the range $\eta \in \left[1\mathrm{e}{-4} , 1\mathrm{e}{-3}\right]$ and 
$\rho \in \left[1\mathrm{e}{-7} , 1\mathrm{e}{-5}\right]$ across our experiments.
We found weight initialization from a normal distribution with stdev=0.02 to perform better than
Glorot \cite{glorot2010understanding} or He \cite{he2015delving} initialization for both Fisher GAN and WGAN-GP.
This initialization is the default in pytorch, while in the WGAN-GP codebase He init \cite{he2015delving} is used.
Specifically the initialization of the generator is more important.

We used some L2 weight decay: $1\mathrm{e}{-6}$ on $\omega$ (i.e. all layers except last) and $1\mathrm{e}{-3}$ weight decay on the last layer $v$.

\subsection{Inception score WGAN-GP baselines: comparison of architecture and weight initialization}
\label{appendix:inception}
As noted in Figure \ref{fig:inception} and in above paragraph, we used intialization from a normal distribution with stdev=0.02
for the inception score experiments for both Fisher GAN and WGAN-GP.
For transparency, and to show that our architecture and initialization benefits both Fisher GAN and WGAN-GP,
we provide plots of different combinations below (Figure \ref{fig:inception_wgangp}). 
Architecture-wise, F64 refers to the architecture described in Appendix \ref{sec:mdl2} with 64 feature maps after the first convolutional layer.
F128 is the architecture from the WGAN-GP codebase \cite{gulrajani2017improved}, which has double the number of feature maps (128 fmaps)
and does not have the two extra layers in G and D (D layers 2-7, G layers 9-14).
The result reported in the WGAN-GP paper \cite{gulrajani2017improved} corresponds to \verb$WGAN-GP F128 He init$.
For WGAN (Figure \ref{fig:inception_wgan}) the 64-fmap architecture gives some initial instability but catches up to the same level as the 128-fmap architecture.

\begin{figure}[H]
\centering
  \includegraphics[width=0.70\textwidth,valign=t]{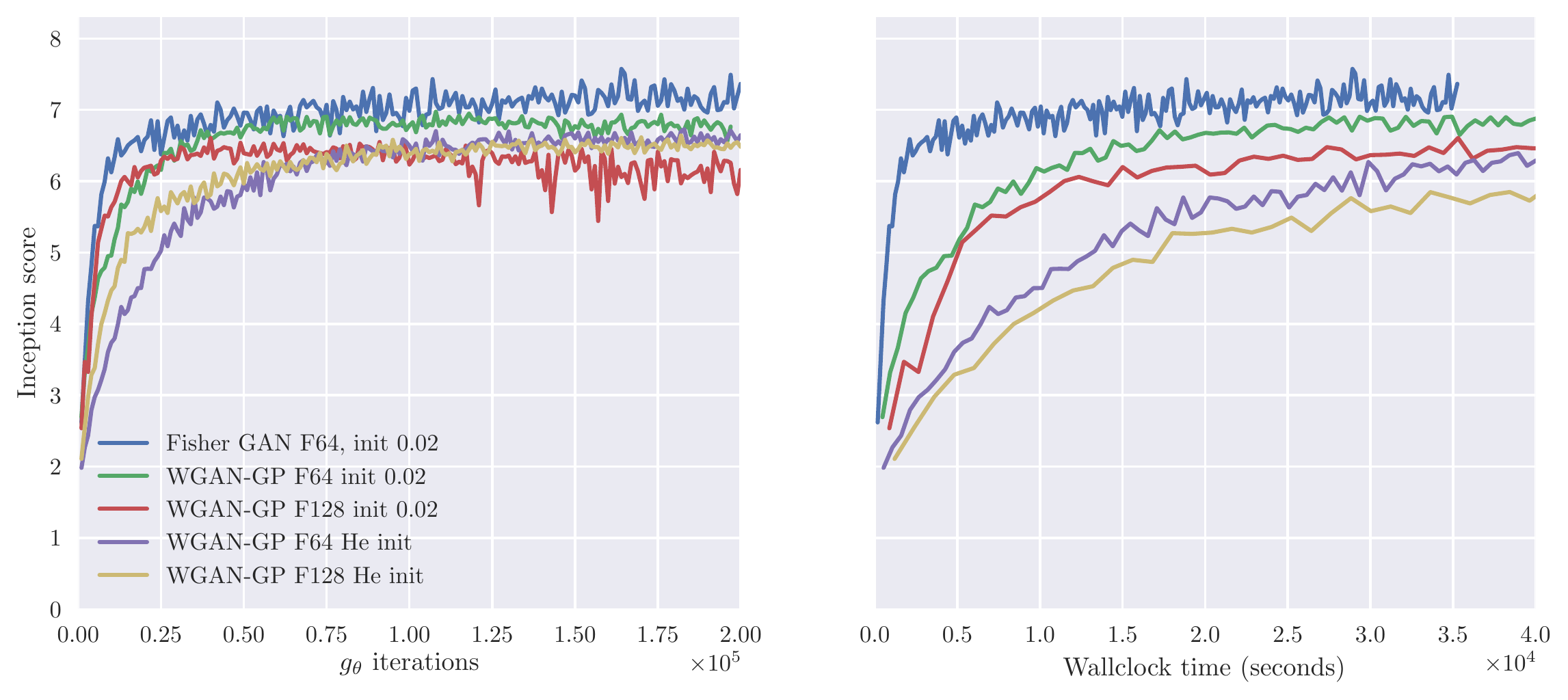}
  \caption{Architecture and initialization variations, trained with WGAN-GP. Fisher included for comparison.
    In the main text (Figure \ref{fig:inception}) we only compare against the best architecture F64 init 0.02.}
  \label{fig:inception_wgangp}
\end{figure}

\begin{figure}[H]
\centering
  \includegraphics[width=0.70\textwidth,valign=t]{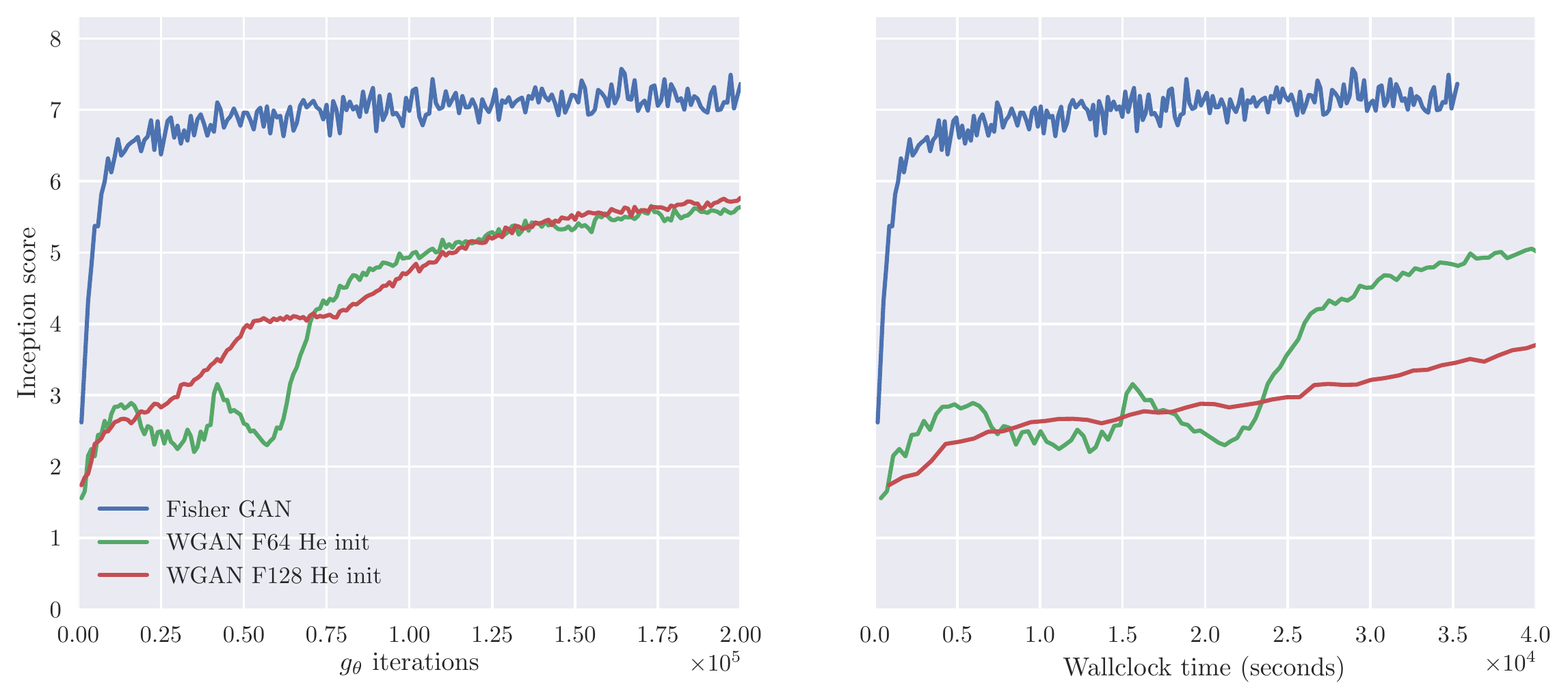}
  \caption{Architecture variations, trained with WGAN. Fisher included for comparison.}
  \label{fig:inception_wgan}
\end{figure}

\subsection{LSUN and CelebA.}
\label{sec:mdl1}
\begin{lstlisting}
### LSUN and CelebA: 64x64 dcgan with G_extra_layers=2 and D_extra_layers=0
G (
  (main): Sequential (
    (0): ConvTranspose2d(100, 512, kernel_size=(4, 4), stride=(1, 1), bias=False)
    (1): BatchNorm2d(512, eps=1e-05, momentum=0.1, affine=True)
    (2): ReLU (inplace)
    (3): ConvTranspose2d(512, 256, kernel_size=(4, 4), stride=(2, 2), padding=(1, 1), bias=False)
    (4): BatchNorm2d(256, eps=1e-05, momentum=0.1, affine=True)
    (5): ReLU (inplace)
    (6): ConvTranspose2d(256, 128, kernel_size=(4, 4), stride=(2, 2), padding=(1, 1), bias=False)
    (7): BatchNorm2d(128, eps=1e-05, momentum=0.1, affine=True)
    (8): ReLU (inplace)
    (9): ConvTranspose2d(128, 64, kernel_size=(4, 4), stride=(2, 2), padding=(1, 1), bias=False)
    (10): BatchNorm2d(64, eps=1e-05, momentum=0.1, affine=True)
    (11): ReLU (inplace)
    (12): Conv2d(64, 64, kernel_size=(3, 3), stride=(1, 1), padding=(1, 1), bias=False)
    (13): BatchNorm2d(64, eps=1e-05, momentum=0.1, affine=True)
    (14): ReLU (inplace)
    (15): Conv2d(64, 64, kernel_size=(3, 3), stride=(1, 1), padding=(1, 1), bias=False)
    (16): BatchNorm2d(64, eps=1e-05, momentum=0.1, affine=True)
    (17): ReLU (inplace)
    (18): ConvTranspose2d(64, 3, kernel_size=(4, 4), stride=(2, 2), padding=(1, 1), bias=False)
    (19): Tanh ()
  )
)
D (
  (main): Sequential (
    (0): Conv2d(3, 64, kernel_size=(4, 4), stride=(2, 2), padding=(1, 1), bias=False)
    (1): LeakyReLU (0.2, inplace)
    (2): Conv2d(64, 128, kernel_size=(4, 4), stride=(2, 2), padding=(1, 1), bias=False)
    (3): BatchNorm2d(128, eps=1e-05, momentum=0.1, affine=True)
    (4): LeakyReLU (0.2, inplace)
    (5): Conv2d(128, 256, kernel_size=(4, 4), stride=(2, 2), padding=(1, 1), bias=False)
    (6): BatchNorm2d(256, eps=1e-05, momentum=0.1, affine=True)
    (7): LeakyReLU (0.2, inplace)
    (8): Conv2d(256, 512, kernel_size=(4, 4), stride=(2, 2), padding=(1, 1), bias=False)
    (9): BatchNorm2d(512, eps=1e-05, momentum=0.1, affine=True)
    (10): LeakyReLU (0.2, inplace)
  )
  (V): Linear (8192 -> 1)
)

\end{lstlisting}
\subsection{CIFAR-10: Sample Quality and Inceptions Scores Experiments}
\label{sec:mdl2}
\begin{lstlisting}
### CIFAR-10: 32x32 dcgan with G_extra_layers=2 and D_extra_layers=2. For samples and inception.
G (
  (main): Sequential (
    (0): ConvTranspose2d(100, 256, kernel_size=(4, 4), stride=(1, 1), bias=False)
    (1): BatchNorm2d(256, eps=1e-05, momentum=0.1, affine=True)
    (2): ReLU (inplace)
    (3): ConvTranspose2d(256, 128, kernel_size=(4, 4), stride=(2, 2), padding=(1, 1), bias=False)
    (4): BatchNorm2d(128, eps=1e-05, momentum=0.1, affine=True)
    (5): ReLU (inplace)
    (6): ConvTranspose2d(128, 64, kernel_size=(4, 4), stride=(2, 2), padding=(1, 1), bias=False)
    (7): BatchNorm2d(64, eps=1e-05, momentum=0.1, affine=True)
    (8): ReLU (inplace)
    (9): Conv2d(64, 64, kernel_size=(3, 3), stride=(1, 1), padding=(1, 1), bias=False)
    (10): BatchNorm2d(64, eps=1e-05, momentum=0.1, affine=True)
    (11): ReLU (inplace)
    (12): Conv2d(64, 64, kernel_size=(3, 3), stride=(1, 1), padding=(1, 1), bias=False)
    (13): BatchNorm2d(64, eps=1e-05, momentum=0.1, affine=True)
    (14): ReLU (inplace)
    (15): ConvTranspose2d(64, 3, kernel_size=(4, 4), stride=(2, 2), padding=(1, 1), bias=False)
    (16): Tanh ()
  )
)
D (
  (main): Sequential (
    (0): Conv2d(3, 64, kernel_size=(4, 4), stride=(2, 2), padding=(1, 1), bias=False)
    (1): LeakyReLU (0.2, inplace)
    (2): Conv2d(64, 64, kernel_size=(3, 3), stride=(1, 1), padding=(1, 1), bias=False)
    (3): BatchNorm2d(64, eps=1e-05, momentum=0.1, affine=True)
    (4): LeakyReLU (0.2, inplace)
    (5): Conv2d(64, 64, kernel_size=(3, 3), stride=(1, 1), padding=(1, 1), bias=False)
    (6): BatchNorm2d(64, eps=1e-05, momentum=0.1, affine=True)
    (7): LeakyReLU (0.2, inplace)
    (8): Conv2d(64, 128, kernel_size=(4, 4), stride=(2, 2), padding=(1, 1), bias=False)
    (9): BatchNorm2d(128, eps=1e-05, momentum=0.1, affine=True)
    (10): LeakyReLU (0.2, inplace)
    (11): Conv2d(128, 256, kernel_size=(4, 4), stride=(2, 2), padding=(1, 1), bias=False)
    (12): BatchNorm2d(256, eps=1e-05, momentum=0.1, affine=True)
    (13): LeakyReLU (0.2, inplace)
  )
  (V): Linear (4096 -> 1)
  (S): Linear (6144 -> 10)
)


\end{lstlisting}

\subsection{CIFAR-10: SSL Experiments}
\label{sec:mdl3}
\begin{lstlisting}
### CIFAR-10: 32x32 D is in the flavor OpenAI Improved GAN, ALI. 
G same as above.

D (
  (main): Sequential (
    (0): Dropout (p = 0.2)
    (1): Conv2d(3, 96, kernel_size=(3, 3), stride=(1, 1), padding=(1, 1))
    (2): LeakyReLU (0.2, inplace)
    (3): Conv2d(96, 96, kernel_size=(3, 3), stride=(1, 1), padding=(1, 1), bias=False)
    (4): BatchNorm2d(96, eps=1e-05, momentum=0.1, affine=True)
    (5): LeakyReLU (0.2, inplace)
    (6): Conv2d(96, 96, kernel_size=(3, 3), stride=(2, 2), padding=(1, 1), bias=False)
    (7): BatchNorm2d(96, eps=1e-05, momentum=0.1, affine=True)
    (8): LeakyReLU (0.2, inplace)
    (9): Dropout (p = 0.5)
    (10): Conv2d(96, 192, kernel_size=(3, 3), stride=(1, 1), padding=(1, 1), bias=False)
    (11): BatchNorm2d(192, eps=1e-05, momentum=0.1, affine=True)
    (12): LeakyReLU (0.2, inplace)
    (13): Conv2d(192, 192, kernel_size=(3, 3), stride=(1, 1), padding=(1, 1), bias=False)
    (14): BatchNorm2d(192, eps=1e-05, momentum=0.1, affine=True)
    (15): LeakyReLU (0.2, inplace)
    (16): Conv2d(192, 192, kernel_size=(3, 3), stride=(2, 2), padding=(1, 1), bias=False)
    (17): BatchNorm2d(192, eps=1e-05, momentum=0.1, affine=True)
    (18): LeakyReLU (0.2, inplace)
    (19): Dropout (p = 0.5)
    (20): Conv2d(192, 384, kernel_size=(3, 3), stride=(1, 1), bias=False)
    (21): BatchNorm2d(384, eps=1e-05, momentum=0.1, affine=True)
    (22): LeakyReLU (0.2, inplace)
    (23): Dropout (p = 0.5)
    (24): Conv2d(384, 384, kernel_size=(3, 3), stride=(1, 1), bias=False)
    (25): BatchNorm2d(384, eps=1e-05, momentum=0.1, affine=True)
    (26): LeakyReLU (0.2, inplace)
    (27): Dropout (p = 0.5)
    (28): Conv2d(384, 384, kernel_size=(1, 1), stride=(1, 1), bias=False)
    (29): BatchNorm2d(384, eps=1e-05, momentum=0.1, affine=True)
    (30): LeakyReLU (0.2, inplace)
    (31): Dropout (p = 0.5)
  )
  (V): Linear (6144 -> 1)
  (S): Linear (6144 -> 10)
)
\end{lstlisting}

\section{Sample implementation in PyTorch}
This minimalistic sample code is based on
\url{https://github.com/martinarjovsky/WassersteinGAN} at commit d92c503.

Some elements that could be added are:
\bit
\it Validation loop
\it Monitoring of weights and activations
\it Separate weight decay for last layer $v$ (we trained with $1\mathrm{e}{-3}$ weight decay on $v$).
\it Adding Cross-Entropy objective and class-conditioned generator.
\eit

\subsection{Main loop}
First note the essential change in the critic's forward pass definition:
\begin{lstlisting}[basicstyle=\tiny\ttfamily,language=Python,firstline=183,lastline=185]
-        output = output.mean(0)
-        return output.view(1)
+        return output.view(-1)
\end{lstlisting}

Then the main training loop becomes:

\begin{lstlisting}[basicstyle=\tiny\ttfamily,language=Python,firstline=155,lastline=226]
gen_iterations = 0
for epoch in range(opt.niter):
    data_iter = iter(dataloader)
    i = 0
    while i < len(dataloader):
        ############################
        # (1) Update D network
        ###########################
        for p in netD.parameters(): # reset requires_grad
            p.requires_grad = True # they are set to False below in netG update

        # train the discriminator Diters times
        if opt.hiDiterStart and (gen_iterations < 25 or gen_iterations % 500 == 0):
            Diters = 100
        else:
            Diters = opt.Diters
        j = 0
        while j < Diters and i < len(dataloader):
            j += 1

            data = data_iter.next()
            i += 1

            # train with real
            real_cpu, _ = data
            netD.zero_grad()
            batch_size = real_cpu.size(0)

            if opt.cuda:
                real_cpu = real_cpu.cuda()
            input.resize_as_(real_cpu).copy_(real_cpu)
            inputv = Variable(input)

            vphi_real = netD(inputv)

            # train with fake
            noise.resize_(opt.batchSize, nz, 1, 1).normal_(0, 1)
            noisev = Variable(noise, volatile = True) # totally freeze netG
            fake = Variable(netG(noisev).data)
            inputv = fake

            vphi_fake = netD(inputv)
            # NOTE here f = <v,phi>   , but with modified f the below two lines are the
            # only ones that need change. E_P and E_Q refer to Expectation over real and fake.
            E_P_f,  E_Q_f  = vphi_real.mean(), vphi_fake.mean()
            E_P_f2, E_Q_f2 = (vphi_real**2).mean(), (vphi_fake**2).mean()
            constraint = (1 - (0.5*E_P_f2 + 0.5*E_Q_f2))
            # See Equation (9)
            obj_D = E_P_f - E_Q_f + alpha * constraint - opt.rho/2 * constraint**2
            # max_w min_alpha obj_D. Compute negative gradients, apply updates with negative sign.
            obj_D.backward(mone)
            optimizerD.step()
            # artisanal sgd. We minimze alpha so a <- a + lr * (-grad)
            alpha.data += opt.rho * alpha.grad.data
            alpha.grad.data.zero_()

        ############################
        # (2) Update G network
        ###########################
        for p in netD.parameters():
            p.requires_grad = False # to avoid computation
        netG.zero_grad()
        # in case our last batch was the tail batch of the dataloader,
        # make sure we feed a full batch of noise
        noise.resize_(opt.batchSize, nz, 1, 1).normal_(0, 1)
        noisev = Variable(noise)
        fake = netG(noisev)
        vphi_fake = netD(fake)
        obj_G = -vphi_fake.mean() # Just minimize mean difference
        obj_G.backward() # G: min_theta
        optimizerG.step()
        gen_iterations += 1
\end{lstlisting}

\subsection{Full diff from reference}
Note that from the arXiv \LaTeX \, source, the file \verb$diff.txt$ could be used in combination with \verb$git apply$.

\begin{lstlisting}[basicstyle=\tiny\ttfamily,language=Python]
diff --git a/main.py b/main.py
index 7c3e638..e0cae42 100644
--- a/main.py
+++ b/main.py
@@ -34,15 +34,17 @@ parser.add_argument('--cuda'  , action='store_true', help='enables cuda')
 parser.add_argument('--ngpu'  , type=int, default=1, help='number of GPUs to use')
 parser.add_argument('--netG', default='', help="path to netG (to continue training)")
 parser.add_argument('--netD', default='', help="path to netD (to continue training)")
-parser.add_argument('--clamp_lower', type=float, default=-0.01)
-parser.add_argument('--clamp_upper', type=float, default=0.01)
+parser.add_argument('--wdecay', type=float, default=0.000, help='wdecay value for Phi')
 parser.add_argument('--Diters', type=int, default=5, help='number of D iters per each G iter')
+parser.add_argument('--hiDiterStart'  , action='store_true', help='do many D iters at start')
 parser.add_argument('--noBN', action='store_true', help='use batchnorm or not (only for DCGAN)')
 parser.add_argument('--mlp_G', action='store_true', help='use MLP for G')
 parser.add_argument('--mlp_D', action='store_true', help='use MLP for D')
-parser.add_argument('--n_extra_layers', type=int, default=0, help='Number of extra layers on gen and disc')
+parser.add_argument('--G_extra_layers', type=int, default=0, help='Number of extra layers on gen and disc')
+parser.add_argument('--D_extra_layers', type=int, default=0, help='Number of extra layers on gen and disc')
 parser.add_argument('--experiment', default=None, help='Where to store samples and models')
 parser.add_argument('--adam', action='store_true', help='Whether to use adam (default is rmsprop)')
+parser.add_argument('--rho', type=float, default=1e-6, help='Weight on the penalty term for (sigmas -1)**2')
 opt = parser.parse_args()
 print(opt)
 
@@ -60,7 +62,7 @@ cudnn.benchmark = True
 if torch.cuda.is_available() and not opt.cuda:
     print("WARNING: You have a CUDA device, so you should probably run with --cuda")
 
-if opt.dataset in ['imagenet', 'folder', 'lfw']:
+if opt.dataset in ['imagenet', 'folder', 'lfw', 'celeba']:
     # folder dataset
     dataset = dset.ImageFolder(root=opt.dataroot,
                                transform=transforms.Compose([
@@ -94,7 +96,6 @@ nz = int(opt.nz)
 ngf = int(opt.ngf)
 ndf = int(opt.ndf)
 nc = int(opt.nc)
-n_extra_layers = int(opt.n_extra_layers)
 
 # custom weights initialization called on netG and netD
 def weights_init(m):
@@ -106,11 +107,11 @@ def weights_init(m):
         m.bias.data.fill_(0)
 
 if opt.noBN:
-    netG = dcgan.DCGAN_G_nobn(opt.imageSize, nz, nc, ngf, ngpu, n_extra_layers)
+    netG = dcgan.DCGAN_G_nobn(opt.imageSize, nz, nc, ngf, ngpu, opt.G_extra_layers)
 elif opt.mlp_G:
     netG = mlp.MLP_G(opt.imageSize, nz, nc, ngf, ngpu)
 else:
-    netG = dcgan.DCGAN_G(opt.imageSize, nz, nc, ngf, ngpu, n_extra_layers)
+    netG = dcgan.DCGAN_G(opt.imageSize, nz, nc, ngf, ngpu, opt.G_extra_layers)
 
 netG.apply(weights_init)
 if opt.netG != '': # load checkpoint if needed
@@ -120,7 +121,7 @@ print(netG)
 if opt.mlp_D:
     netD = mlp.MLP_D(opt.imageSize, nz, nc, ndf, ngpu)
 else:
-    netD = dcgan.DCGAN_D(opt.imageSize, nz, nc, ndf, ngpu, n_extra_layers)
+    netD = dcgan.DCGAN_D(opt.imageSize, nz, nc, ndf, ngpu, opt.D_extra_layers)
     netD.apply(weights_init)
 
 if opt.netD != '':
@@ -132,6 +133,7 @@ noise = torch.FloatTensor(opt.batchSize, nz, 1, 1)
 fixed_noise = torch.FloatTensor(opt.batchSize, nz, 1, 1).normal_(0, 1)
 one = torch.FloatTensor([1])
 mone = one * -1
+alpha = torch.FloatTensor([0]) # lagrange multipliers
 
 if opt.cuda:
     netD.cuda()
@@ -139,14 +141,16 @@ if opt.cuda:
     input = input.cuda()
     one, mone = one.cuda(), mone.cuda()
     noise, fixed_noise = noise.cuda(), fixed_noise.cuda()
+    alpha = alpha.cuda()
+alpha = Variable(alpha, requires_grad=True)
 
 # setup optimizer
 if opt.adam:
-    optimizerD = optim.Adam(netD.parameters(), lr=opt.lrD, betas=(opt.beta1, 0.999))
-    optimizerG = optim.Adam(netG.parameters(), lr=opt.lrG, betas=(opt.beta1, 0.999))
+    optimizerD = optim.Adam(netD.parameters(), lr=opt.lrD, betas=(opt.beta1, 0.999), weight_decay=opt.wdecay)
+    optimizerG = optim.Adam(netG.parameters(), lr=opt.lrG, betas=(opt.beta1, 0.999), weight_decay=opt.wdecay)
 else:
-    optimizerD = optim.RMSprop(netD.parameters(), lr = opt.lrD)
-    optimizerG = optim.RMSprop(netG.parameters(), lr = opt.lrG)
+    optimizerD = optim.RMSprop(netD.parameters(), lr = opt.lrD, weight_decay=opt.wdecay)
+    optimizerG = optim.RMSprop(netG.parameters(), lr = opt.lrG, weight_decay=opt.wdecay)
 
 gen_iterations = 0
 for epoch in range(opt.niter):
@@ -160,7 +164,7 @@ for epoch in range(opt.niter):
             p.requires_grad = True # they are set to False below in netG update
 
         # train the discriminator Diters times
-        if gen_iterations < 25 or gen_iterations % 500 == 0:
+        if opt.hiDiterStart and (gen_iterations < 25 or gen_iterations % 500 == 0):
             Diters = 100
         else:
             Diters = opt.Diters
@@ -168,10 +172,6 @@ for epoch in range(opt.niter):
         while j < Diters and i < len(dataloader):
             j += 1
 
-            # clamp parameters to a cube
-            for p in netD.parameters():
-                p.data.clamp_(opt.clamp_lower, opt.clamp_upper)
-
             data = data_iter.next()
             i += 1
 
@@ -185,18 +185,28 @@ for epoch in range(opt.niter):
             input.resize_as_(real_cpu).copy_(real_cpu)
             inputv = Variable(input)
 
-            errD_real = netD(inputv)
-            errD_real.backward(one)
+            vphi_real = netD(inputv)
 
             # train with fake
             noise.resize_(opt.batchSize, nz, 1, 1).normal_(0, 1)
             noisev = Variable(noise, volatile = True) # totally freeze netG
             fake = Variable(netG(noisev).data)
             inputv = fake
-            errD_fake = netD(inputv)
-            errD_fake.backward(mone)
-            errD = errD_real - errD_fake
+
+            vphi_fake = netD(inputv)
+            # NOTE here f = <v,phi>   , but with modified f the below two lines are the
+            # only ones that need change. E_P and E_Q refer to Expectation over real and fake.
+            E_P_f,  E_Q_f  = vphi_real.mean(), vphi_fake.mean()
+            E_P_f2, E_Q_f2 = (vphi_real**2).mean(), (vphi_fake**2).mean()
+            constraint = (1 - (0.5*E_P_f2 + 0.5*E_Q_f2))
+            # See Equation (9)
+            obj_D = E_P_f - E_Q_f + alpha * constraint - opt.rho/2 * constraint**2
+            # max_w min_alpha obj_D. Compute negative gradients, apply updates with negative sign.
+            obj_D.backward(mone)
             optimizerD.step()
+            # artisanal sgd. We minimze alpha so a <- a + lr * (-grad)
+            alpha.data += opt.rho * alpha.grad.data
+            alpha.grad.data.zero_()
 
         ############################
         # (2) Update G network
@@ -209,14 +219,20 @@ for epoch in range(opt.niter):
         noise.resize_(opt.batchSize, nz, 1, 1).normal_(0, 1)
         noisev = Variable(noise)
         fake = netG(noisev)
-        errG = netD(fake)
-        errG.backward(one)
+        vphi_fake = netD(fake)
+        obj_G = -vphi_fake.mean() # Just minimize mean difference
+        obj_G.backward() # G: min_theta
         optimizerG.step()
         gen_iterations += 1
 
-        print('[%d/%d][%d/%d][%d] Loss_D: %f Loss_G: %f Loss_D_real: %f Loss_D_fake %f'
+        IPM_enum  = E_P_f.data[0]  - E_Q_f.data[0]
+        IPM_denom = (0.5*E_P_f2.data[0] + 0.5*E_Q_f2.data[0]) ** 0.5
+        IPM_ratio = IPM_enum / IPM_denom
+        print(('[%d/%d][%d/%d][%d] IPM_enum: %.4f IPM_denom: %.4f IPM_ratio: %.4f '
+               'E_P_f: %.4f E_Q_f: %.4f E_P_(f^2): %.4f E_Q_(f^2): %.4f')
             % (epoch, opt.niter, i, len(dataloader), gen_iterations,
-            errD.data[0], errG.data[0], errD_real.data[0], errD_fake.data[0]))
+            IPM_enum, IPM_denom, IPM_ratio,
+            E_P_f.data[0], E_Q_f.data[0], E_P_f2.data[0], E_Q_f2.data[0]))
         if gen_iterations % 500 == 0:
             real_cpu = real_cpu.mul(0.5).add(0.5)
             vutils.save_image(real_cpu, '{0}/real_samples.png'.format(opt.experiment))
diff --git a/models/dcgan.py b/models/dcgan.py
index 1dd8dbf..ea86a94 100644
--- a/models/dcgan.py
+++ b/models/dcgan.py
@@ -48,9 +48,7 @@ class DCGAN_D(nn.Module):
             output = nn.parallel.data_parallel(self.main, input, range(self.ngpu))
         else: 
             output = self.main(input)
-            
-        output = output.mean(0)
-        return output.view(1)
+        return output.view(-1)
 
 class DCGAN_G(nn.Module):
     def __init__(self, isize, nz, nc, ngf, ngpu, n_extra_layers=0):
@@ -148,9 +146,7 @@ class DCGAN_D_nobn(nn.Module):
             output = nn.parallel.data_parallel(self.main, input, range(self.ngpu))
         else: 
             output = self.main(input)
-            
-        output = output.mean(0)
-        return output.view(1)
+        return output.view(-1)
 
 class DCGAN_G_nobn(nn.Module):
     def __init__(self, isize, nz, nc, ngf, ngpu, n_extra_layers=0):
\end{lstlisting}

\end{document}